%% file: main.tex
\documentclass[11pt, dvipsnames, DIV=11]{scrartcl}
\usepackage[english]{babel}

\usepackage[utf8]{inputenc}
\usepackage[T1]{fontenc}
\usepackage{natbib}
\usepackage{subcaption}
\usepackage{booktabs}
\usepackage{multirow}
\usepackage{url}
\usepackage{tikz}
\usetikzlibrary{arrows}
\usetikzlibrary{shadows,arrows.meta}
\usepackage{paralist}
\usepackage[stable]{footmisc}
\usepackage[edges]{forest}
\usepackage{fontawesome5}
\usepackage{ifthen}

\usepackage{bm}
\usepackage{amsmath}
\usepackage{amssymb}
\usepackage{amsthm}
\usepackage{amsfonts}
\usepackage{thmtools}		
\usepackage{mleftright}
\usepackage{stmaryrd}
\usepackage{nicefrac}
\usepackage{algorithm}
\usepackage{algorithmicx}
\usepackage[noend]{algpseudocode}

\let\originalleft\left
\let\originalright\right
\renewcommand{\left}{\mathopen{}\mathclose\bgroup\originalleft}
\renewcommand{\right}{\aftergroup\egroup\originalright}

\newtheorem{theorem}{Theorem}

\newtheorem{proposition}[theorem]{Proposition}

\newtheorem{lemma}[theorem]{Lemma}
\newtheorem{corollary}[theorem]{Corollary}
\theoremstyle{definition}

\usepackage{thm-restate}
\usepackage[mathic=true]{mathtools}
\usepackage{fixmath}
\usepackage{siunitx}

\usepackage{pifont}
\newcommand{\cmark}{\ding{51}}
\newcommand{\xmark}{\ding{55}}

\usepackage{enumitem}
\setlist[enumerate]{itemsep=0.2ex, topsep=0.5\topsep}
\setlist[description]{itemsep=0.2ex, topsep=0.5\topsep}
\setlist[itemize]{itemsep=0.2ex, topsep=0.5\topsep}

\makeatletter
\def\thmt@refnamewithcomma #1#2#3,#4,#5\@nil{%
\@xa\def\csname\thmt@envname #1utorefname\endcsname{#3}%
\ifcsname #2refname\endcsname
\csname #2refname\expandafter\endcsname\expandafter{\thmt@envname}{#3}{#4}%
\fi
}
\makeatother

\usepackage[pagebackref,
pdfa,
hidelinks,
pdftex, 
pdfdisplaydoctitle,
pdfpagelabels,
pdfauthor={Christopher Morris},
pdftitle={WL meet VC},
pdfsubject={},
pdfkeywords={},
pdfproducer={Latex with the hyperref package},
pdfcreator={pdflatex}
]{hyperref}

\usepackage[capitalise,noabbrev]{cleveref}   

\usepackage{microtype}
\usepackage{ellipsis}

\usepackage[scaled=0.86]{helvet}
\usepackage{lmodern}

\usepackage[auth-lg]{authblk}

\newcommand{\fg}[1]{{{\textcolor{magenta}{\textbf{[FG:} {#1}\textbf{]}}}}}

\input{macros}

\recalctypearea

\title{\huge\normalfont\bfseries WL meet VC}
\author[1]{Christopher Morris}
\author[2]{Floris Geerts}
\author[1]{Jan Tönshoff}
\author[1]{Martin Grohe}

\affil[1]{RWTH Aachen University}
\affil[2]{University of Antwerp}
\date{\vspace{-30pt}}

\begin{document}

\maketitle

\begin{abstract}
	Recently, many works studied the expressive power of graph neural networks (GNNs) by linking it to the $1$-dimensional Weisfeiler--Leman algorithm (\wlone). Here, the \wlone{} is a well-studied heuristic for the graph isomorphism problem, which iteratively colors or partitions a graph's vertex set. While this connection has led to significant advances in understanding and enhancing GNNs' expressive power, it does not provide insights into their generalization performance, i.e., their ability to make meaningful predictions beyond the training set. In this paper, we study GNNs' generalization ability through the lens of Vapnik--Chervonenkis (VC) dimension theory in two settings, focusing on graph-level predictions. First, when no upper bound on the graphs' order is known, we show that the bitlength of GNNs' weights tightly bounds their VC dimension. Further, we derive an upper bound for GNNs' VC dimension using the number of colors produced by the \wlone{}. Secondly, when an upper bound on the graphs' order is known, we show a tight connection between the number of graphs distinguishable by the \wlone{} and GNNs' VC dimension. Our empirical study confirms the validity of our theoretical findings.
\end{abstract}

\section{Introduction}
Graph-structured data are prevalent across application domains ranging from chemo- and bioinformatics~\citep{Barabasi2004,Jum+2021,Sto+2020} to image~\citep{Sim+2017} and social-network analysis~\citep{Eas+2010}, indicating the importance of machine learning methods for such data. Nowadays, there are numerous approaches for machine learning for graph-structured, most notably those based on \new{graph kernels}~\citep{Borg+2020,Kri+2019} or \new{graph neural networks} (GNNs)~\citep{Cha+2020,Gil+2017,Mor+2022}. Here, graph kernels~\citep{She+2011} based on the \new{$1$-dimensional Weisfeiler--Leman algorithm} (\wlone)~\citep{Wei+1968}, a well-studied heuristic for the graph isomorphism problem, and corresponding GNNs~\citep{Mor+2019,Xu+2018b}, have recently advanced the state-of-the-art in supervised vertex- and graph-level learning~\citep{Mor+2022}. Further, based on the \new{$k$-dimensional Weisfeiler--Leman algorithm} (\kwl), \wlone's more powerful generalization, several works generalized GNNs to \new{higher-order GNNs} ($k$-GNNs), resulting in provably more expressive architectures, e.g.,~\citet{Azi+2020,geerts2022,Mar+2019,Mor+2019, Morris2020b,Mor+2022,Mor+2022b}.

While devising provably expressive GNN-like architectures is a meaningful endeavor, it only partially addresses the challenges of machine learning with graphs. That is, expressiveness results reveal little about an architecture's ability to generalize to graphs outside the training set. Surprisingly, only a few notable contributions study GNNs' generalization behaviors, e.g.,~\citet{Gar+2020,Kri+2018,Lia+2021,Mas+2022,Sca+2018}. However, these approaches express GNN's generalization ability using only classical graph parameters, e.g., maximum degree, number of vertices, or edges, which cannot fully capture the complex structure of real-world graphs. Further, most approaches study generalization in the \new{non-uniform regime}, i.e., assuming that the GNNs operate on graphs of a pre-specified order. Further, they only investigate the case $k=1$, i.e., standard GNNs, ignoring more expressive generalizations; see the previous paragraph.

This paper investigates the influence of graph structure and the parameters' encoding lengths on GNNs' generalization by tightly connecting \wlone's expressivity and GNNs' Vapnik--Chervonenkis (VC) dimension. Specifically, we show that:
\begin{enumerate}
	\item In the non-uniform regime, we prove \emph{tight} bounds on GNNs' VC dimension. We show that  GNNs' VC dimension depends tightly on the number of equivalence classes computed by the \wlone{} over a set of graphs; see~\cref{thm:colorbound_up,prop:matchingvc}. Moreover, our results easily extend to the \kwl and many recent expressive GNN extensions.
	\item In the uniform regime, i.e., when graphs can have arbitrary order, we show that GNNs' VC dimension is \emph{lower} and \emph{upper bounded} by the largest bitlength of its weights; see~\cref{thm:bl_lowerr}.
	\item In both the uniform and non-uniform regimes, GNNs' VC dimension depends \textit{logarithmically  on the number of colors} computed by the \wlone{} and polynomially on the number of parameters; see~\cref{thm:bartlett}.
	\item Empirically, we show that our theoretical findings hold in practice.
\end{enumerate}
Overall, our results provide new insights into GNNs' generalization behavior and how graph structure and parameters influence it. Specifically, our results imply that a complex graph structure, captured by \wlone{}, results in worse generalization performance. The same holds for increasing the encoding length of the GNN's parameters. \emph{Importantly, our theory provides the first link between expressivity results and generalization ability.} Moreover, our results establish the \emph{first} lower bounds for GNNs' VC dimension. See~\cref{fig:overview} for a high-level overview of our results.

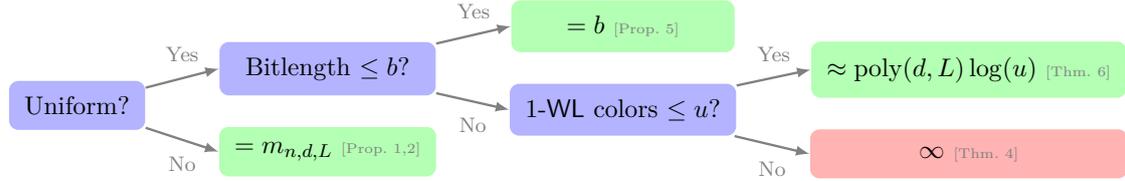
\begin{figure}
	\centering
	\pgfkeys{/pgf/inner sep=0.6em}
	\scalebox{0.88}{
		\begin{forest}
			for tree={
			grow=0,reversed,
			edge={thick, color=gray,line width=1pt},
			edge+={->,>=latex},
			child anchor=west,
			rounded corners,
			fill=green!30,
			l sep+=15pt,
			minimum height=21pt
			}
			[Uniform?, fill=blue!30
			[Bitlength $\leq b$?, fill=blue!30, minimum width=92pt, edge label={node[midway,above,font=\footnotesize]{Yes}},fill=blue!30
			[${= b}$ {\tiny \textcolor{gray}{[Prop.~\ref{thm:bl_lowerr}]}}, minimum width=95pt, edge label={node[midway,above,font=\footnotesize]{Yes}}]
			[\wlone{} colors $\leq u$?, fill=blue!30, edge label={node[midway,below,font=\footnotesize]{No}}
				[$\approx\text{poly}({d,L})\log(u)$ {\tiny \textcolor{gray}{[Thm.~\ref{thm:bartlett}]}}, edge label={node[midway,above,font=\footnotesize]{Yes}}]
				[$\infty$ {\tiny \textcolor{gray}{[Thm.~\ref{thm:unboundedbitss}]}}, minimum width=135pt, edge label={node[midway,below,font=\footnotesize]{No}}, fill=red!30]
			]
			]
			[${=  m_{n,d,L}}$ {\tiny \textcolor{gray}{[Prop.~\ref{thm:colorbound_up},\ref{prop:matchingvc}]}},minimum width=83pt, edge label={node[midway,below,font=\footnotesize]{No}}]
			]
			]
		\end{forest}}
	\caption{Overview of our results for bounded-width GNNs. Green and red boxes denote VC dimension bounds. Here, $m_{n,d,L}$ denotes the number of graphs of order at most $n$ with boolean $d$-dimensional features distinguishable by \wlone{} after $L$ iterations.}\label{fig:overview}
\end{figure}

\subsection{Related work}
In the following, we discuss relevant related work.

\paragraph{GNNs} Recently, GNNs~\citep{Gil+2017,Sca+2009} emerged as the most prominent graph representation learning architecture. Notable instances of this architecture include, e.g.,~\citet{Duv+2015,Ham+2017}, and~\citet{Vel+2018}, which can be subsumed under the message-passing framework introduced in~\citet{Gil+2017}. In parallel, approaches based on spectral information were introduced in, e.g.,~\citet{Bru+2014,Defferrard2016,Gam+2019,Kip+2017,Lev+2019}, and~\citet{Mon+2017}---all of which descend from early work in~\citet{bas+1997,Gol+1996,Kir+1995,Mer+2005,mic+2005,mic+2009,Sca+2009}, and~\citet{Spe+1997}.

\paragraph{Limits of GNNs and more expressive architectures}
Recently, connections between GNNs and Weisfeiler--Leman type algorithms have been shown~\citep{Bar+2020,Gee+2020a,Mor+2019,Xu+2018b}. Specifically,~\citet{Mor+2019} and~\citet{Xu+2018b} showed that the \wlone{} limits the expressive power of any possible GNN architecture in terms of distinguishing non-isomorphic graphs. In turn, these results have been generalized to the \kwl, see, e.g.,~\citet{Azi+2020,Gee+2020,Mar+2019,Mor+2019,Morris2020b,Mor+2022b}, and connected to permutation-equivariant functions approximation over graphs, see, e.g.,~\citet{Che+2019,Mae+2019,Azi+2020,geerts2022}. Further,~\citet{Aam+2022} devised an improved analysis using randomization. Recent works have extended the expressive power of GNNs, e.g., by encoding vertex identifiers~\citep{Mur+2019b, Vig+2020}, using random features~\citep{Abb+2020,Das+2020,Sat+2020}, equivariant graph polynomials~\citep{Pun+2023}, homomorphism and subgraph counts~\citep{Bar+2021,botsas2020improving,Hoa+2020}, spectral information~\citep{Bal+2021}, simplicial~\citep{Bod+2021} and cellular complexes~\citep{Bod+2021b}, persistent homology~\citep{Hor+2021}, random walks~\citep{Toe+2021,Mar+2022}, graph decompositions~\citep{Tal+2021}, relational~\citep{Bar+2022}, distance~\citep{li2020distance} and directional information~\citep{beaini2020directional}, subgraph information~\citep{Bev+2021,Cot+2021,Fen+2022,Fra+2022,Hua+2022,Mor+2022,Pap+2021,Pap+2022,Qia+2022,Thi+2021,Wij+2022,You+2020,Zha+2021,Zha+2022,Zha+2023b}, and biconnectivity~\citep{Zha+2023}. See~\citet{Mor+2022} for an in-depth survey on this topic. \citet{geerts2022} devised a general approach for bounding the expressive power of a large variety of GNNs utilizing the \wlone{} or \kwl{}. Recently,~\citet{Kim+2022} showed that transformer architectures~\citep{Mue+2023} can simulate the $2$-$\mathsf{WL}$ \citet{Gro+2023} showed tight connections between GNNs' expressivity and circuit complexity. Moreover,~\citet{Ros+2023} investigated the expressive power of different aggregation functions beyond sum aggregation.

\paragraph{GNN's generalization capabilities}
\citet{Sca+2018} used classical techniques from learning theory~\citep{Kar+1997} to show that GNNs' VC dimension~\citep{Vap+95} with piece-wise polynomial activation functions on a \emph{fixed} graph,
under various assumptions, is in $\cO(P^2n\log n)$, where $P$ is the number of parameters and $n$ is the order of the input graph. We note here that~\citet{Sca+2018} analyzed a different type of GNN not aligned with modern GNN architectures~\citep{Gil+2017}. \citet{Gar+2020} showed that the empirical Rademacher complexity, e.g.,~\citep{Moh+2012}, of a specific, simple GNN architecture, using sum aggregation, is bounded in the maximum degree, the number of layers, Lipschitz constants of activation functions, and parameter matrices' norms. We note here that their analysis assumes weight sharing across layers. \citet{Lia+2021} refined these results via a PAC-Bayesian approach, further refined in~\citet{Ju+2023}. \citet{Mas+2022} used random graphs models to show that GNNs' generalization ability depends on the (average) number of vertices in the resulting graphs. \citet{Ver+2019} studied the generalization abilities of $1$-layer GNNs in a transductive setting based on algorithmic stability. Similarly,~\citet{Ess+2021} used stochastic block models to study the transductive Rademacher complexity~\citep{Yan+2009,Tol+2016} of standard GNNs. Moreover,~\cite{Kri+2018} leveraged results from graph property testing~\citep{Gol2010} to study the sample complexity of learning to distinguish various graph properties, e.g., planarity or triangle freeness, using graph kernels~\citep{Borg+2020,Kri+2019}. We stress that all of the above approaches only consider classical graph parameters to bound the generalization abilities of GNNs. Finally,~\cite{Yeh+2021} showed negative results for GNNs' ability to generalize to larger graphs.

\medskip
\noindent See~\cref{app:related_work} for an overview of the Weisfeiler--Leman algorithm's theoretical properties. However, the generalization properties of GNNs and their connection to expressivity is understood to a lesser extent.  

\section{Preliminaries}\label{prelim_ext}
Let $\Nb \coloneqq \{ 1,2,3, \dots \}$. For $n \geq 1$, let $[n] \coloneqq \{ 1, \dotsc, n \} \subset \Nb$. We use $\{\!\!\{ \dots\}\!\!\}$ to denote multisets, i.e., the generalization of sets allowing for multiple instances for each of its elements.

\paragraph{Graphs} A \new{graph} $G$ is a pair $(V(G),E(G))$ with \emph{finite} sets of
\new{vertices} or \new{nodes} $V(G)$ and \new{edges} $E(G) \subseteq \{ \{u,v\} \subseteq V(G) \mid u \neq v \}$. If not otherwise stated, we set $n \coloneqq |V(G)|$, and the graph is of \new{order} $n$. We also call the graph $G$ an $n$-order graph. For ease of notation, we denote the edge $\{u,v\}$ in $E(G)$ by $(u,v)$ or $(v,u)$. In the case of \new{directed graphs}, the set $E(G) \subseteq \{ (u,v) \in V(G) \times V(G) \mid u \neq v \}$ and a \new{directed acyclic graph} (DAG) is a directed graph with no directed cycles. A \new{(vertex-)labeled graph} $G$ is a triple $(V(G),E(G),\ell)$ with a (vertex-)label function $\ell \colon V(G) \to \Nb$. Then $\ell(v)$ is a \new{label} of $v$, for $v$ in $V(G)$. An \new{attributed graph} $G$  is a triple $(V(G),E(G),a)$ with a graph $(V(G),E(G))$ and (vertex-)attribute function $a \colon V(G) \to \Rb^{1 \times d}$, for some $d > 0$. That is, contrary to labeled graphs, we allow for vertex annotations from an uncountable set. Then $a(v)$ is an \new{attribute} or \new{feature} of $v$, for $v$ in $V(G)$. Equivalently, we define an $n$-order attributed graph $G \coloneqq (V(G),E(G),a)$ as a pair $\mG=(G,\mLG)$, where $G = (V(G),E(G))$ and $\mLG$ in $\Rb^{n\times d}$ is a \new{vertex feature matrix}. Here, we identify $V(G)$ with $[n]$. For a matrix $\mLG$ in $\Rb^{n\times d}$ and $v$ in $[n]$, we denote by $\mLG_{v \cdot}$ in $\Rb^{1\times d}$ the $v$th row of $\mLG$ such that $\mL_{v \cdot} \coloneqq a(v)$. We also write $\Rb^d$ for $\Rb^{1\times d}$.

The \new{neighborhood} of $v$ in $V(G)$ is denoted by $N(v) \coloneqq  \{ u \in V(G) \mid (v, u) \in E(G) \}$ and the \new{degree} of a vertex $v$ is  $|N(v)|$. In case of directed graphs, $N^+(u) \coloneqq  \{ v \in V(G) \mid (v, u) \in E(G) \}$ and $N^-(u) \coloneqq  \{ v \in V(G) \mid (u, v) \in E(G) \}$. The \new{in-degree} and \new{out-degree} of a vertex $v$ are $|N^+(v)|$ and $|N^-(v)|$, respectively. Two graphs $G$ and $H$ are \new{isomorphic} and we write $G \simeq H$ if there exists a bijection $\varphi \colon V(G) \to V(H)$ preserving the adjacency relation, i.e., $(u,v)$ is in $E(G)$ if and only if
$(\varphi(u),\varphi(v))$ is in $E(H)$. Then $\varphi$ is an \new{isomorphism} between
$G$ and $H$. In the case of labeled graphs, we additionally require that
$l(v) = l(\varphi(v))$ for $v$ in $V(G)$, and similarly for attributed graphs. 

\subsection{The Weisfeiler--Leman algorithm}
We here describe the \wlone{} and refer to~\cref{kwl} for the \kwl. The \wlone{} or color refinement is a well-studied heuristic for the graph isomorphism problem, originally proposed by~\citet{Wei+1968}.\footnote{Strictly speaking, the \wlone{} and color refinement are two different algorithms. That is, the \wlone{} considers neighbors and non-neighbors to update the coloring, resulting in a slightly higher expressive power when distinguishing vertices in a given graph; see~\cite {Gro+2021} for details. For brevity, we consider both algorithms to be equivalent.}
Intuitively, the algorithm determines if two graphs are non-isomorphic by iteratively coloring or labeling vertices. Given an initial coloring or labeling of the vertices of both graphs, e.g., their degree or application-specific information, in
each iteration, two vertices with the same label get different labels if the number of identically labeled neighbors is unequal. These labels induce a vertex partition, and the algorithm terminates when after some iteration, the algorithm does not refine the current partition, i.e., when a \new{stable coloring} or \new{stable partition} is obtained. Then, if the number of vertices annotated with a specific label is different in both graphs, we can conclude that the two graphs are not isomorphic. It is easy to see that the algorithm cannot distinguish all non-isomorphic graphs~\citep{Cai+1992}. Nonetheless, it is a powerful heuristic that can successfully test isomorphism for a broad class of graphs~\citep{Bab+1979}.

Formally, let $G = (V(G),E(G),\ell)$ be a labeled graph. In each iteration, $t > 0$, the \wlone{} computes a vertex coloring $C^1_t \colon V(G) \to \Nb$, depending on the coloring of the neighbors. That is, in iteration $t>0$, we set
\begin{equation*}
	C^1_t(v) \coloneqq \REL\Big(\!\big(C^1_{t-1}(v),\oms C^1_{t-1}(u) \mid u \in N(v)  \cms \big)\! \Big),
\end{equation*}
for all vertices $v$ in $V(G)$,
where $\REL$ injectively maps the above pair to a unique natural number, which has not been used in previous iterations. In iteration $0$, the coloring $C^1_{0}\coloneqq \ell$. To test if two graphs $G$ and $H$ are non-isomorphic, we run the above algorithm in ``parallel'' on both graphs. If the two graphs have a different number of vertices colored $c$ in $\Nb$ at some iteration, the \wlone{} \new{distinguishes} the graphs as non-isomorphic. Moreover, if the number of colors between two iterations, $t$ and $(t+1)$, does not change, i.e., the cardinalities of the images of $C^1_{t}$ and $C^1_{i+t}$ are equal, or, equivalently,
\begin{equation*}
	C^1_{t}(v) = C^1_{t}(w) \iff C^1_{t+1}(v) = C^1_{t+1}(w),
\end{equation*}
for all vertices $v$ and $w$ in $V(G)$, the algorithm terminates. For such $t$, we define the \new{stable coloring}
$C^1_{\infty}(v) = C^1_t(v)$, for $v$ in $V(G)$. The stable coloring is reached after at most $\max \{ |V(G)|,|V(H)| \}$ iterations~\citep{Gro2017}. We define the \emph{color complexity}
of a graph $G$ as the number of colors computed by the \wlone{} after $|V(G)|$ iterations on $G$.

Due to the shortcomings of the $\wlone$ or color refinement in distinguishing non-isomorphic
graphs, several researchers, e.g.,~\citet{Bab1979,Cai+1992},
devised a more powerful generalization of the former, today known
as the $k$-dimensional Weisfeiler-Leman algorithm, operating on $k$-tuples of vertices rather than single vertices; see~\cref{kwl} for details.

\subsection{Graph Neural Networks}
\label{sec:gnn}
Intuitively, GNNs learn a vectorial representation, i.e., a $d$-dimensional real-valued vector, representing each vertex in a graph by aggregating information from neighboring vertices. Formally, let $G = (V(G),E(G),\ell)$ be a labeled graph with initial vertex features $\hb_{v}^\tup{0}$ in $\Rb^{d}$ that are \emph{consistent} with $\ell$. That is, each vertex $v$ is annotated with a feature  $\hb_{v}^\tup{0}$ in $\Rb^{d}$ such that $\hb_{v}^\tup{0} = \hb_{u}^\tup{0}$ if and only $\ell(v) = \ell(u)$, e.g., a one-hot encoding of the labels $\ell(u)$ and $\ell(v)$. Alternatively,  $\hb_{v}^\tup{0}$ can be an attribute or a feature of the vertex $v$, e.g., physical measurements in the case of chemical molecules. A GNN architecture consists of a stack of neural network layers, i.e., a composition of permutation-equivariant parameterized functions. Similarly to the \wlone, each layer aggregates local neighborhood information, i.e., the neighbors' features around each vertex, and then passes this aggregated information on to the next layer.

Following, \citet{Gil+2017} and \citet{Sca+2009}, in each layer, $t > 0$,  we compute vertex features
\begin{equation}\label{def:gnn}
	\hb_{v}^\tup{t} \coloneqq
	\UPD^\tup{t}\Bigl(\hb_{v}^\tup{t-1},\AGG^\tup{t} \bigl(\oms \hb_{u}^\tup{t-1}
	\mid u\in N(v) \cms \bigr)\Bigr) \in \Rb^{d},
\end{equation}
where  $\UPD^\tup{t}$ and $\AGG^\tup{t}$ may be differentiable parameterized functions, e.g., neural networks.\footnote{Strictly speaking, \citet{Gil+2017} consider a slightly more general setting in which vertex features are computed by $\hb_{v}^\tup{t} \coloneqq
		\UPD^\tup{t}\Bigl(\hb_{v}^\tup{t-1},\AGG^\tup{t} \bigl(\oms (\hb_v^\tup{t-1},\hb_{u}^\tup{t-1},\ell(v,u))
		\mid u\in N(v) \cms \bigr)\Bigr)$, where $\ell(v,u)$ denotes the edge label of the edge $(v,u)$.}
In the case of graph-level tasks, e.g., graph classification, one uses
\begin{equation}\label{def:readout}
	\hb_G \coloneqq \RO\bigl( \oms \hb_{v}^{\tup{L}}\mid v\in V(G) \cms \bigr) \in \Rb,
\end{equation}
to compute a single vectorial representation based on learned vertex features after iteration $L$.\footnote{For simplicity, we assume GNNs to return scalars on graphs. This makes the definition of VC dimension more concise.}
Again, $\RO$  may be a differentiable parameterized function. To adapt the parameters of the above three functions, they are optimized end-to-end, usually through a variant of stochastic gradient descent, e.g.,~\citet{Kin+2015}, together with the parameters of a neural network used for classification or regression. See~\cref{kgnn} for a definition of (higher-order) $k$-GNNs.

\paragraph{Notation} In the subsequent sections, we use the following notation. We denote the class of all (labeled) graphs by $\cG$, the class of all graphs with $d$-dimensional, real-valued vertex features by $\cG_{d}$, the class of all graphs with $d$-dimensional boolean vertex features by $\cG_{d}^\bool$, the class of all graphs with an order of at most $n$ and $d$-dimensional vertex features by $\cG_{d,n}$, and the class of all graphs with $d$-dimensional vertex features
and of color complexity at most $u$ by $\cG_{d,\leq u}$.

Further, we consider the following classes of GNNs.  We denote the class of all GNNs consisting of $L$ layers with $(L+1)$th layer readout layer by  $\GNN(L)$, the subset of $\GNN(L)$ but whose aggregate, update and readout
functions have a width at most $d$ by $\GNN(d,L)$, and the subset of $\GNN(L)$ but whose aggregation function is a summation and update and readout functions are single layer perceptrons of width at most $d$ by $\GNN_{\mathsf{slp}}(d,L)$. More generally, we consider the class $\GNN_{\mathsf{mlp}}(d,L)$ of GNNs using summation  for aggregation and such that update and readout functions are \new{multilayer perceptrons} (MLPs), all of width of at most $d$. We refer to elements in $\GNN_{\mathsf{mlp}}(d,L)$ as \textit{simple GNNs}. See~\cref{sec:simpleGNNs} for details. We stress that simple GNNs are already expressive enough to be equivalent to the \wlone{} in distinguishing non-isomorphic graphs.

\paragraph{VC dimension of GNNs}
For a class $\cC$ of GNNs and $\cX$ of graphs, $\vcdim_\cX(\cC)$ is the maximal number $m$ of graphs $\mG_1,\ldots, \mG_m$ in $\cX$ that
can be shattered by $\cC$. Here, $\mG_1,\ldots,\mG_m$ are \new{shattered} if for any $\pmb\tau $ in $\{0,1\}^m$ there exists a GNN $\gnn$ in $\cC$
such that for all $i$ in $[m]$:
\begin{equation}\label{threshold}
	\gnn(\mG_i)=
	\begin{cases}
		\geq 2/3 & \text{if $\tau_i=1$, and} \\
		\leq 1/3 & \text{if $\tau_i=0$.}
	\end{cases}
\end{equation}
The above definition can straightforwardly be generalized to $k$-GNNs. Bounding the VC dimension directly implies an upper bound on the generalization error; see~\cref{vc_error} and~\cite{Vap+95,Moh+2012} for details.

\paragraph{Bitlength of GNNs}
Below we study the dependence of GNNs' VC dimension on the \new{bitlength} of its weights. Assume an $L$-layered GNN with a set of parameters $\Theta$, then the GNN's bitlength is the maximum number of bits needed to encode each weight in $\Theta$ and the parameters specifying the activation functions. We define the bitlength of a class of GNNs as the maximum bitlength across all GNNs in the class.

\section{WL meet VC \faHandshake[regular]}
We first consider the non-uniform regime, i.e., we assume an upper bound on the graphs' order. Given the connection between GNNs and the \wlone{}~\citep{Mor+2019,Xu+2018}, GNNs' ability to shatter a set of graphs can easily be related to distinguishability by the \wlone{}. Indeed, let $\cS$ be
a collection of graphs that GNNs can shatter. Hence, for each pair of graphs in $\cS$, we have a GNN that distinguishes them. By the results of~\cite{Mor+2019,Xu+2018b}, this implies that the graphs in $\cS$ are pairwise \wlone{} distinguishable. In other words, when considering the VC dimension of GNNs on a class $\cS$ of graphs with a bounded number $m$ of \wlone{} distinguishable graphs, then $m$ is also an upper bound on the VC dimension of GNNs on graphs in $\cS$. For example, let us first consider the VC dimension of GNNs on the class $\cG_{d,n}^\bool$ consisting of graphs of an order of at most $n$ with $d$-dimensional boolean features. Let $m_{n,d,L}$ be the maximal number of graphs in $\cG_{d,n}^\bool$ distinguishable by \wlone{} after $L$ iterations. Then, the same argument as above implies that $m_{n,d,L}$ is also the maximal number of graphs in $\cG_{d,n}^\bool$  that can be shattered by $L$-layer GNNs, as is stated next.
\begin{proposition}\label{thm:colorbound_up}
	For all $n$, $d$ and $L$, it holds that
	\begin{equation*}
		\vcdim_{\cG_{d,n}^{\bool}}\!\!\bigl(\GNN(L)\bigr) \leq m_{n,d,L}.
	\end{equation*}
\end{proposition}
This upper bound holds regardless of the choice of aggregation, update, and readout functions used in the GNNs. We next show a matching lower bound for the VC dimension of GNNs on graphs in $\cG_{d,n}^{\bool}$.
In fact, the lower bound already holds for simple GNNs of width $\cO(nm_{n,d,L})$.
\begin{proposition}\label{prop:matchingvc}
	For all $n$, $d$, and $L$,  all $m_{n,d,L}$ \wlone-distinguishable graphs of order at most $n$ with $d$-dimensional boolean features can be shattered by sufficiently
	wide $L$-layer GNNs. Hence,
	\begin{equation*}
		\vcdim_{\cG_{d,n}^\bool}\!\!\bigl(\GNN(L)\bigr)=m_{n,d,L}.
	\end{equation*}
\end{proposition}
The lower bound follows from the fact that GNNs are as powerful as the \wlone{} \citep{Mor+2019,Xu+2018}.
Indeed, \citet{Mor+2019} showed that $L$ iterations of the \wlone{} on graphs in $\cG_{d,n}^{\bool}$ can be simulated by a simple $L$-layered GNN of width $\cO(n)$. To shatter all $m_{n,d,L}$ graphs, we first simulate
\wlone{} on all $m_{n,d, L}$ graphs combined using a simple $L$-layer GNN of width in $\cO(nm_{n,d, L})$. We then define a readout layer whose weights can be used to shatter the input graphs based on the computed \wlone{} vertex colors in the graphs. We note that the above two results can be straightforwardly generalized to $k$-GNNs, i.e., their VC dimension is tightly connected to the expressive power of the \kwl, and other recent extensions of GNNs; see~\cref{klift_app}.

We now consider the uniform regime, i.e., we assume no upper bound on the graphs' order. Since the number $m_{n,d,L}$ of \wlone{} distinguishable graphs increases for growing $n$,~\cref{prop:matchingvc} implies that the VC dimension of $L$-layered GNNs on the class $\cG_d^\bool$ of all graphs with $d$-dimensional boolean features but of arbitrary order is unbounded.
\begin{corollary}\label{infty_simple}
	For all $d$ and $L\geq 1$, it holds that $\vcdim_{\cG_d^\bool}(\GNN(L))=\infty$.
\end{corollary}
The proof of this result requires update and readout functions in GNNs of \emph{unbounded width}. Using a different  ``bit extraction'' proof technique, we can strengthen the previous result such that \emph{fixed-width} GNNs can be considered.
\begin{theorem}\label{thm:unboundedbitss}
	For all $d,L$ at least two, it holds  $\vcdim_{\cG_{d}^\bool}(\GNN(d,L))=\infty$.
\end{theorem}
Again, this result holds even for the class of simple GNNs. The theorem relies on the following result, which is of independent interest.
\begin{proposition}\label{thm:bl_lowerr}
	There exists a family $\cF_b$ of simple $2$-layer GNNs of width two and bitlength $\cO(b)$ using piece-wise linear activation functions such that its VC dimension is \emph{exactly} $b$.
\end{proposition}
That is, to show that the VC dimension is infinite for GNNs in 
$\GNN(d,L)$ with $d$ and $L \geq 2$, we leverage~\cref{thm:bl_lowerr}, implying that we can shatter an arbitrary number of graphs by such GNNs, provided that they have bit precision $\cO(b)$. The GNNs in~\cref{thm:unboundedbitss} have arbitrary precision reals, so they can also shatter these 
graphs. Since this works for any $b$,~\cref{thm:unboundedbitss} follows.
 
\cref{thm:bl_lowerr} in turn is proved as follows. For the upper bound, we observe that there are only exponentially many (in $b$) GNNs in $\cF_b$, from which the upper bound immediately follows. Indeed, classical VC theory implies a  bound on the VC dimension for finite classes of GNNs, logarithmic in the number of GNNs in the class. The lower bound proof is more challenging and requires constructing the collection $\cF_b$ of simple GNNs that can shatter $b$ graphs belonging to $\cG_{d}^\bool$. We remark that the $b$ graphs used are of order $\cO(2^b)$ and have $\cO(2^b)$ \wlone{} vertex colors.
 
Finally, we show bounded VC dimension when both the width and the number of layers of GNNs, \textit{and} input graphs are restricted. We obtain the bound by leveraging state-of-the-art VC dimension bounds for feedforward neural networks (FNNs) by~\citet{Bar+2019}. To control the number of parameters, we consider the class $\GNN_{\mathsf{slp}}(d,L)$ of simple GNNs in which update and readout functions are single-layer perceptrons of bounded width $d$. We specify this class of GNNs using $P = \nw$  parameters and the choice of activation functions in the perceptrons. Regarding the class of input graphs, we consider the class $\cG_{d,\leq u}$ consisting of graphs having $d$-dimensional features and color complexity at most $u$. Note that we only bound the number of colors appearing in a single graph and not the number of colors appearing in all graphs of the class. Intuitively, the parameter $u$ comes into play because we can reduce input graphs by combining \wlone{} equivalent vertices. The reduced graph has $u$ vertices, and edges have weights. One can run extended GNNs, taking edge weights into account on the reduced graph without loss of information. Moreover, one can tie extended GNNs to FNNs. Hence, the parameter $n$ can be replaced by $u$ when analyzing the VC dimension of the GNN-related FNNs.

We now relate $\cG_{d,n}$ and $\cG_{d,\leq u}$. Since
any graph of order at most $n$ has at most $n$ \wlone{} colors,
$\cG_{d,n}\subseteq \cG_{d,\leq n}$. The bound below thus also complements the upper bound on $\cG_{d,n}^\bool$
given earlier, but now for fixed-width GNNs. However,  $\cG_{d,\leq u}$ may contain graphs of arbitrary order. For example, all regular graphs (of the same degree) belong to $\cG_{d,\leq 1}$. We also remark that
there is no upper bound on the number of \wlone{} distinguishable graphs for $\cG_{d,\leq u}$, because we only bound the number of colors appearing in a single graph and not the number of colors appearing in all graphs of the class. For example, all regular graphs of arbitrary degrees are in $\cG_{0,\le 1}$, but regular graphs of different degrees can be distinguished.
As such, the bounds obtained earlier do not apply. Finally, in the bound below, input graphs can have real features.
\begin{theorem}\label{thm:bartlett}
	Assume $d$ and $L$ in $\Nb$, and GNNs in  $\GNN_{\mathsf{slp}}(d,L)$ using piece-wise polynomial activation functions with $p > 0$ pieces and degree $\delta \geq 0$. Let $P = \nw$ be the number of parameters in the GNNs. For all $u$ in $\Nb$,
	\begin{equation*}
		\vcdim_{\cG_{d,\leq u}}(\GNN_{\mathsf{slp}}(d,L))\leq
		\begin{cases}
			\cO(P\log(puP))                   & \text{if $\delta=0$,} \\
			\cO(LP\log(puP))                  & \text{if $\delta=1$,} \\
			\cO(LP\log(puP)+L^2P\log(\delta)) & \text{if $\delta>1$}. \\
		\end{cases}
	\end{equation*}
\end{theorem}
We note that the above result can be straightforwardly generalized to $k$-GNNs. These upper bounds, concerning the dependency on $u$, cannot be improved by more than a constant factor. Indeed,
for the $b$ graphs used in \cref{thm:bl_lowerr}, $u=\cO(2^b)$. Moreover, the simple GNNs in $\cF_b$ belong to
$\GNN_{\mathsf{slp}}(2,2)$ and use piecewise-linear activation functions with $p=4$ and $\delta=1$. Since they shatter $b=\cO(\log(u))$ graphs, this matches with the upper bound up to constant factors.

\paragraph{Discussion} Our results include the first lower bounds of GNNs' VC dimension and the first inherent connection between GNNs' VC dimension and the expressive power of the \wlone{}. In particular, we find that a larger bit precision leads to a higher VC dimension and sample complexity. Therefore, our findings provide an additional incentive for reducing the precision of GNNs besides the already-known benefits of reduced training and inference time. Furthermore, we show that the VC dimension bounds for GNNs can be tightened when considering graphs of low color complexity. This connection to the number of colors indicates that graphs with complex structures, captured by the \wlone{}, may have a higher VC dimension. Moreover, if the \wlone{} can distinguish a large set of graphs of a given dataset, our results imply that a sufficiently expressive GNN will require a large set of training samples to generalize well. Therefore, we can use the \wlone{} to assess GNNs' generalization ability on a given dataset quickly. The same relation holds for $k$-GNNs and the \kwl. Moreover, our results extend easily to recent, more expressive GNN enhancements, e.g., subgraph-based~\citep{botsas2020improving} or subgraph-enhanced GNNs~\citep{Bev+2021,Qia+2022}. Hence, our results lead to a better understanding of how expressivity influences the learning performance of modern GNN architectures; see~\cref{klift_app}.

\section{Limitations, possible road maps, and future work}
Although the results are the first ones explicitly drawing a tight connection between expressivity and generalization, there are still many open questions. First, excluding~\cref{thm:colorbound_up}, the results investigate specific GNN classes using sum aggregation.
Hence, the results should be extended to specific GNN layers commonly used in practice, such as that in~\citet{Xu+2018b}, and the effect of different aggregation functions, such as max or mean, should be studied in detail. Moreover, the results only give meaningful results for discretely labeled graphs. Hence, the results should be extended to attributed graphs. Secondly, although the experimental results in~\cref{sec:experiments} suggest that our VC dimension bounds hold in practice to some extent, it is well known that they do not explain the generalization behavior of deep neural networks in the over-parameterized regime, as shown in~\citet{Bar+2019a}, trained with variants of stochastic gradient descent. Therefore, it is a future challenge to understand how graph structure influences the generalization properties of over-parameterized GNNs, trained with variants of stochastic gradient descent and what role the Weisfeiler--Leman algorithm plays in this context.

\section{Experimental evaluation}\label{sec:experiments}
In the following, we investigate how well the VC dimension bounds from the previous section hold in practice. Specifically, we answer the following questions.
\begin{description}
	\item[Q1] How does the number of parameters influence GNNs' generalization performance?
	\item[Q2] How does the number of \wlone-distinguishable graphs influence GNNs' generalization performance?
	\item[Q3] How does the bitlength influence a GNN's ability to fit random data?
\end{description}

The source code of all methods and evaluation procedures is available at \url{https://www.github.com/chrsmrrs/wl_vs_vs}.

\paragraph{Datasets} To investigate questions \textbf{Q1} and \textbf{Q2}, we used the datasets \textsc{Enzymes}~\citep{Bor+2005a,Sch+2004}, \textsc{MCF-7}~\citep{Yan+2008}, \textsc{MCF-7H}~\citep{Yan+2008}, \textsc{Mutagenicity}~\citep{Kaz+2005,Rie+2008}, and \textsc{NCI1} and \textsc{NCI109}~\citep{Wal+2008,She+2011} provided by~\citet{Mor+2020}. See~\cref{statistics} for dataset statistics and properties. For question \textbf{Q3}, to investigate the influence of bitlength on GNN's VC dimension, we probed how well GNNs can fit random data. Hence, the experiments on these datasets aim at empirically verifying the VC dimension bounds concerning bitlength. To that, we created a synthetic dataset; see~\cref{syn_app}. Since it is challenging to simulate different bitlengths without specialized hardware, we resorted to simulating an increased bitlength via an increased feature dimension; see~\cref{bit_dim}.

All experiments are therefore conducted with standard 32-bit precision. 
We also experimented with 64-bit precision but observed no clear difference. 
Furthermore, 16-bit precision proved numerically unstable in this setting. 

\paragraph{Neural architectures} For the experiments regarding \textbf{Q1} and \textbf{Q2}, we used the simple GNN layer described in~\cref{simple_app} using a $\relu$ activation function, ignoring possible edge labels. To answer question \textbf{Q1}, we fixed the number of layers to five and chose the feature dimension $d$ in $\{ 4,16,256,1\,024 \}$. To answer \textbf{Q2}, we set the feature dimension $d$ to $64$ and choose the number of layers from $\{0,\dots,6\}$. We used sum pooling and a two-layer MLP  for all experiments for the final classification. To investigate \textbf{Q3}, we used the architecture described in~\cref{syn_gnn_app}. In essence, we used a 2-layer MLP for the message generation function in each GNN layer and added batch normalization~\citep{ioffe2015batch} before each non-linearity and fixed the number of layers to $3$, and varied the feature dimension $d$ in $\{ 4,16,64,256 \}$.

\paragraph{Experimental protocol and model configuration} For the experiments regarding \textbf{Q1} and \textbf{Q2}, we uniformly and at random choose 90\% of a dataset for training and the remaining 10\% for testing. We repeated each experiment five times  and report mean test accuracies and standard deviations. We optimized the standard cross entropy loss for 500 epochs using the \textsc{Adam} optimizer~\citep{Kin+2015}. Moreover, we used a learning rate of 0.001 across all experiments and no learning rate decay or dropout. For \textbf{Q3}, we set the learning rate to $10^{-4}$ and the number of epochs to 100\,000, and repeated each experiment 50 times. All architectures were implemented using \textsc{PyTorch Geometric}~\citep{Fey+2019} and executed on a workstation with 128GB RAM and an NVIDIA Tesla V100 with 32GB memory.

\begin{table}[t]
	\caption{Train and test classification accuracies using different numbers of parameters, using five layers, studying how the number of parameters influences generalization.}
	\label{tud_weights}
	\centering

	\resizebox{0.8\textwidth}{!}{ 	\renewcommand{\arraystretch}{1.05}
		\begin{tabular}{@{}c <{\enspace}@{}lcccccc@{}} \toprule
			\multirow{3}{*}{\vspace*{4pt}\textbf{Dimension}} & \multirow{3}{*}{\vspace*{4pt}\textbf{Split}} & \multicolumn{6}{c}{\textbf{Dataset}}                                                                                                                                                \\\cmidrule{3-8}
			                                                 &                                              & {\textsc{Enzymes}}                   & {\textsc{MCF-7}}           & {\textsc{MCF-7H}}          &
			{\textsc{Mutagenicity}}                          & {\textsc{NCI1}}                              &
			{\textsc{NCI109}}                                                                                                                                                                                                                                                                     \\	\toprule
			\multirow{2}{*}{4}                               & Train                                        & 31.0 \scriptsize	$\pm 3.1$            & 92.1 \scriptsize	$\pm 0.4$  & 92.1 \scriptsize	$\pm 0.4$  & 79.7 \scriptsize	$\pm 0.9$  & 76.7 \scriptsize	$\pm 7.3$ & 66.4 \scriptsize	$\pm 9.5$ \\
			                                                 & Test                                         & 25.3 \scriptsize	$\pm 5.2$            & 92.4  \scriptsize	$\pm 0.2$ & 92.3  \scriptsize	$\pm 0.3$ & 75.8 \scriptsize	$\pm 0.9$  & 72.4 \scriptsize	$\pm 7.1$ & 63.6 \scriptsize	$\pm 9.1$ \\
			\cmidrule{1-8}
			\multirow{2}{*}{16}                              & Train                                        & 76.8 \scriptsize	$\pm 6.4$            & 96.0 \scriptsize	$\pm 0.1$  & 96.5  \scriptsize	$\pm 0.3$ & 92.7 \scriptsize	$\pm 2.7$  & 88.4 \scriptsize	$\pm 6.2$ & 86.4 \scriptsize	$\pm 0.9$ \\
			                                                 & Test                                         & 41.7 \scriptsize	$\pm 9.4$            & 93.2 \scriptsize	$\pm 0.5$  & 93.1  \scriptsize	$\pm 0.4$ & 79.8 \scriptsize	$\pm 2.2$  & 76.1 \scriptsize	$\pm 2.1$ & 78.6 \scriptsize	$\pm 1.9$ \\
			\cmidrule{1-8}
			\multirow{2}{*}{256}                             & Train                                        & 98.2 \scriptsize	$\pm 3.6$            & 99.7 \scriptsize	$<0.1$     & 99.9 \scriptsize	$<0.1$     & 100.0 \scriptsize	$\pm 0.0$ & 99.8 \scriptsize	$<0.1$    & 97.5 \scriptsize	$\pm 2.1$ \\
			                                                 & Test                                         & 54.7 \scriptsize	$\pm 2.4$            & 94.0  \scriptsize	$\pm 0.2$ & 93.6 \scriptsize	$\pm 0.2$  & 80.7 \scriptsize	$\pm 1.0$  & 81.8 \scriptsize	$\pm 1.5$ & 82.1 \scriptsize	$\pm 1.0$ \\
			\cmidrule{1-8}
			\multirow{2}{*}{1\,024}                          & Train                                        & 99.8 \scriptsize	$\pm 0.2$            & 99.8 \scriptsize	$<0.1$     & 99.8 \scriptsize	$\pm 0.1$  & 99.9 \scriptsize	$\pm 0.2$  & 99.8 \scriptsize	$\pm 0.1$ & 98.6 \scriptsize	$\pm 1.0$ \\
			                                                 & Test                                         & 54.3 \scriptsize	$\pm $2.3            & 93.8  \scriptsize	$\pm 0.3$ & 93.6 \scriptsize	$\pm 0.2$  & 81.7 \scriptsize	$\pm 0.8$  & 80.5 \scriptsize	$\pm 1.0$ & 82.9 \scriptsize	$\pm 0.9$ \\
			\bottomrule
		\end{tabular}
	}
\end{table}

\begin{figure}
	\centering
	\subcaptionbox{\textsc{Enzymes}}{\includegraphics[scale=0.45]{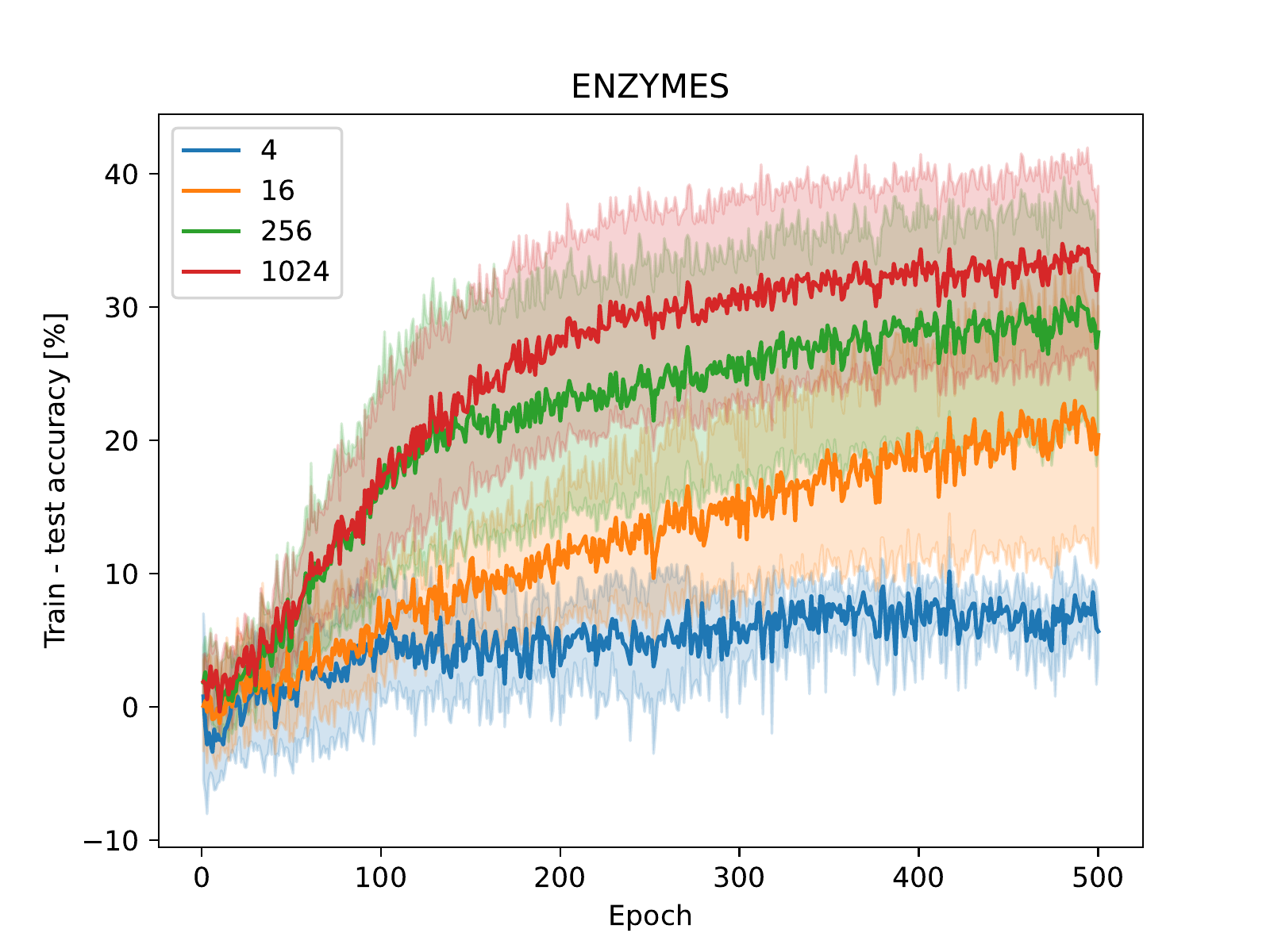}}%
	\hfill
	\subcaptionbox{\textsc{MCF-7}}{\includegraphics[scale=0.45]{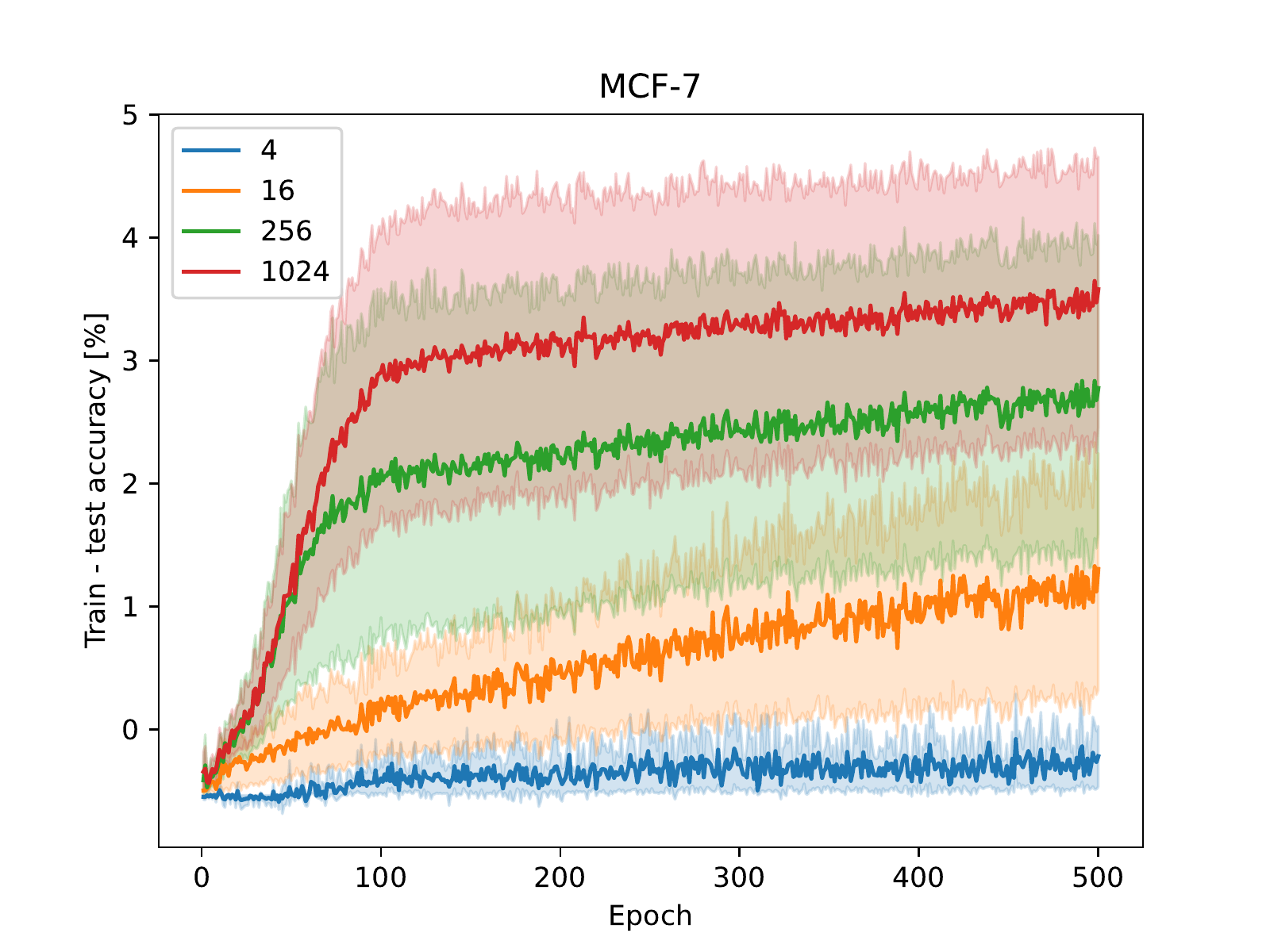}}\\
	\subcaptionbox{\textsc{MCF-7H}}{\includegraphics[scale=0.45]{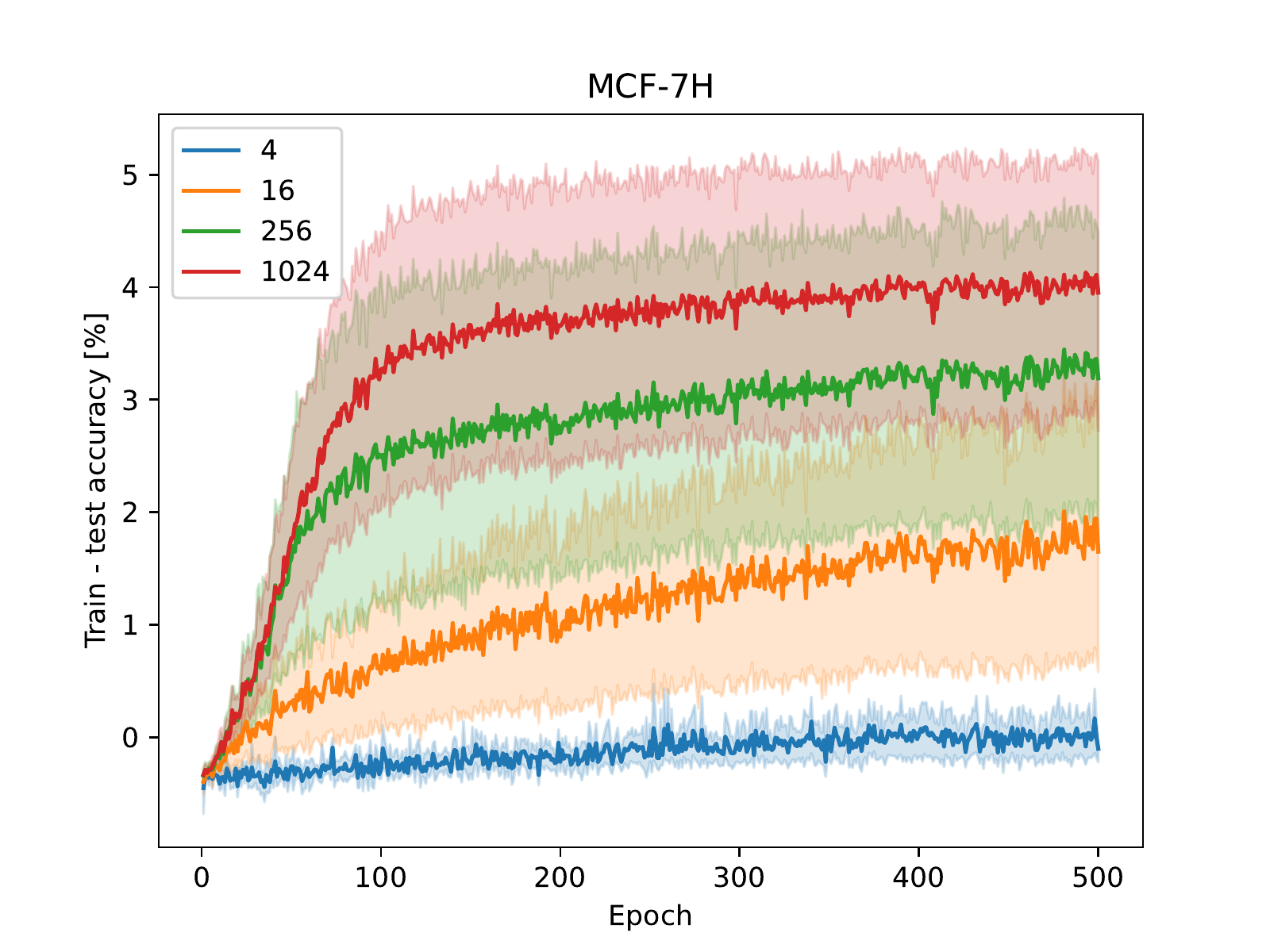}}%
	\hfill
	\subcaptionbox{\textsc{Mutagenicity}}{\includegraphics[scale=0.45]{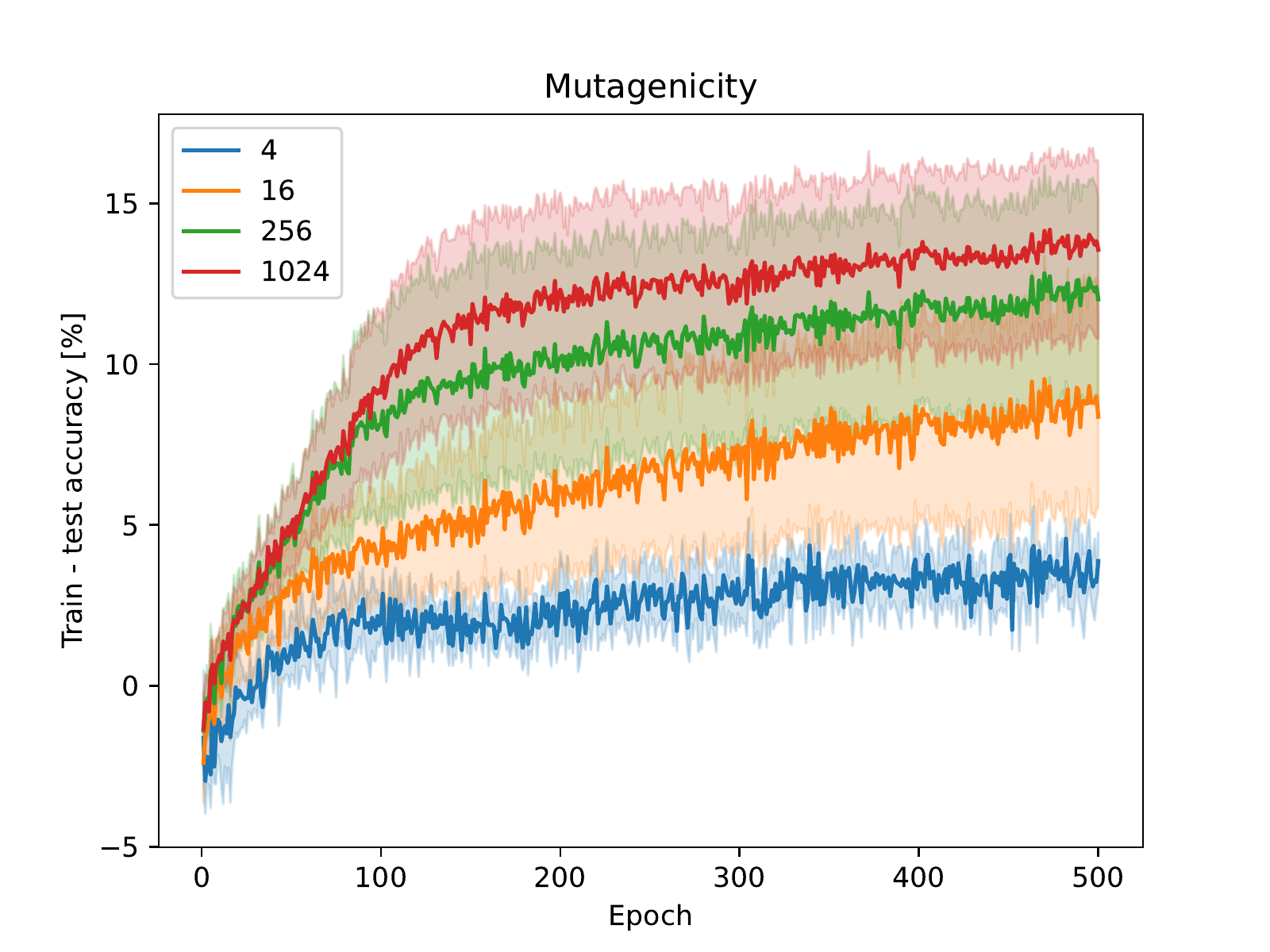}}\\
	\subcaptionbox{\textsc{NCI1}}{\includegraphics[scale=0.45]{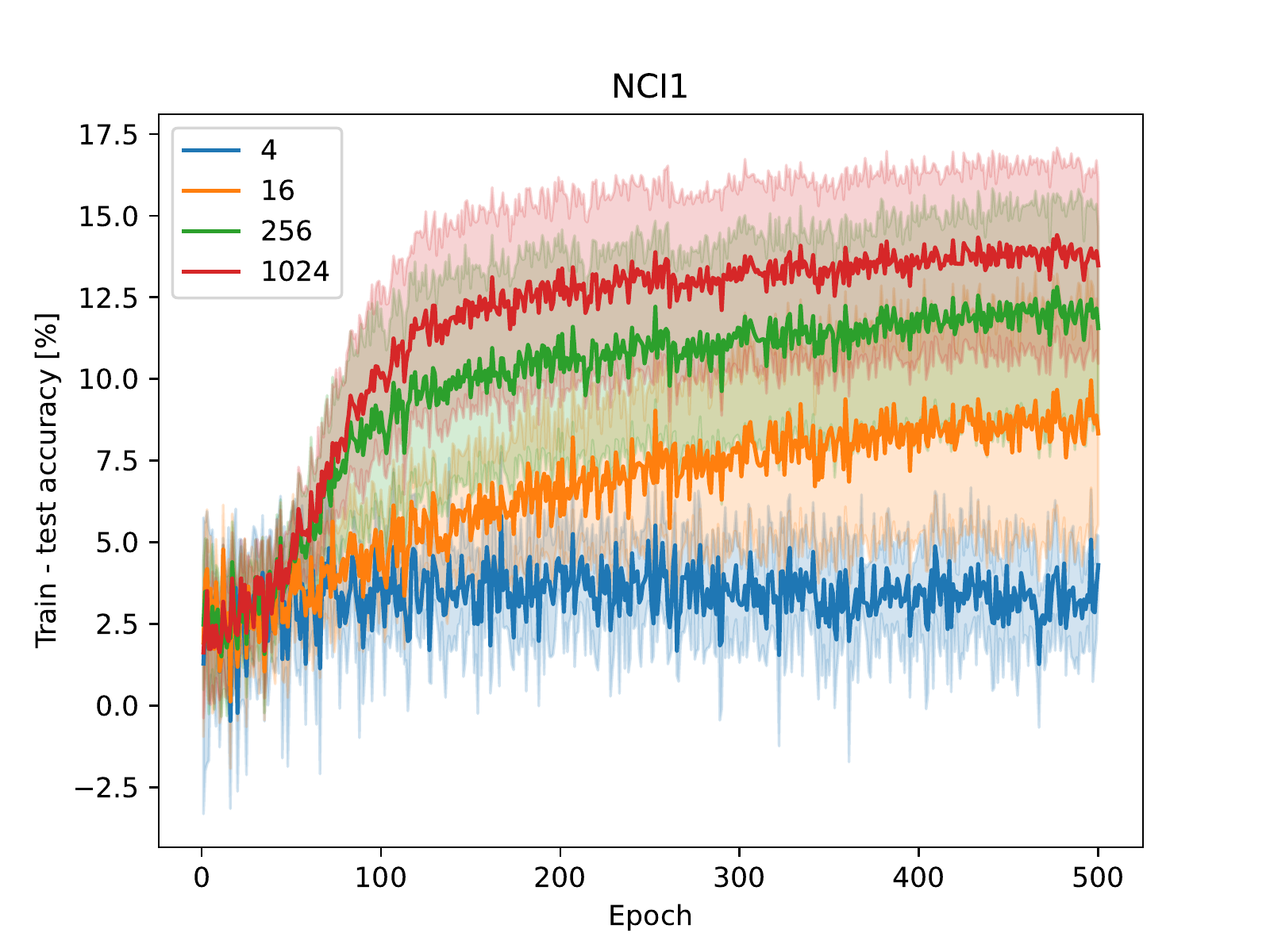}}%
	\hfill
	\subcaptionbox{\textsc{NCI109}}{\includegraphics[scale=0.45]{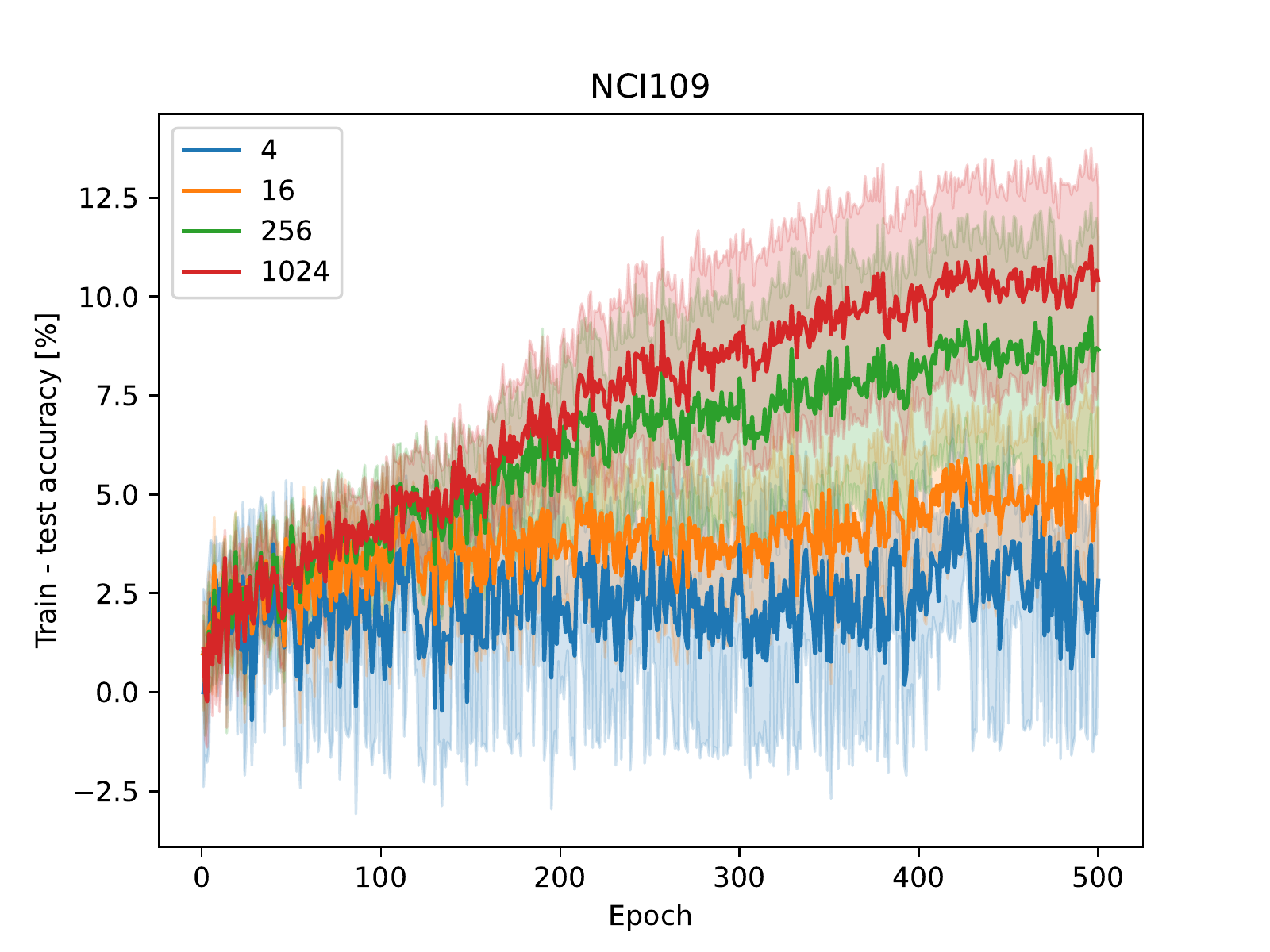}}%
	\caption{Difference between train and test accuracy for different feature dimensions in $\{ 4, 16, 256, 1\,024\}$.}\label{fig:exp1}
\end{figure}

\begin{table}[t]
	\caption{Train and test classification accuracies using different numbers of layers and a feature dimension of $64$, studying how the number of different color histograms (number of \wlone-distinguishable graphs $m_{n,d,L}$) influences generalization.}
	\label{colors}
	\centering

	\resizebox{0.8\textwidth}{!}{ 	\renewcommand{\arraystretch}{1.05}
		\begin{tabular}{@{}c <{\enspace}@{}lcccccc@{}} \toprule
			\multirow{3}{*}{\vspace*{4pt}\textbf{Layers}} & \multirow{3}{*}{\vspace*{4pt}\textbf{Split}} & \multicolumn{6}{c}{\textbf{Dataset}}                                                                                                                                                  \\\cmidrule{3-8}
			                                              &                                              & {\textsc{Enzymes}}                   & {\textsc{MCF-7}}           & {\textsc{MCF-7H}}          &
			{\textsc{Mutagenicity}}                       & {\textsc{NCI1}}                              &
			{\textsc{NCI109}}                                                                                                                                                                                                                                                                    \\	\toprule
			\multirow{4}{*}{0}                            & Train                                        & 40.7 \scriptsize	$\pm 0.5$            & 91.7 \scriptsize	$\pm 0.1$  & 91.8 \scriptsize	$\pm 0.1$  & 77.2 \scriptsize	$ \pm 0.3$  & 74.5 \scriptsize	$\pm 0.3$  & 73.1 \scriptsize	$\pm 0.5$ \\
			                                              & Test                                         & 33.7 \scriptsize	$\pm 1.6$            & 91.9 \scriptsize	$< 0.1$    & 91.2 \scriptsize	$\pm 0.1$  & 75.7 \scriptsize	$\pm 1.2$   & 67.9 \scriptsize	$\pm 1.3 $ & 71.5 \scriptsize	$\pm 0.7$ \\
			                                              & Difference                                   & 7.0 \scriptsize	$\pm 1.9$             & -0.2 \scriptsize	$\pm 0.1$  & 1.0   \scriptsize	$\pm 0.1$ & 1.5 \scriptsize	$\pm 1.3$    & 6.5 \scriptsize	$\pm 1.2$   & 1.6 \scriptsize	$\pm 0.6$  \\
			                                              & \# Histograms                                & 385                                  & 11\,533                    & 19\,625                    & 2\,819                      & 2\,889                     & 2\,929                    \\
			\cmidrule{1-8}
			\multirow{4}{*}{1}                            & Train                                        & 66.7  \scriptsize	$\pm 3.6$           & 91.8  \scriptsize	$\pm 0.1$ & 92.1 \scriptsize	$<0.1$     & 90.9 \scriptsize $ \pm 0.1$ & 92.0 \scriptsize	$\pm 1.5$  & 83.4 \scriptsize	$\pm 1.5$ \\
			                                              & Test                                         & 52.3 \scriptsize	$\pm 5.0$            & 91.9 \scriptsize	$<0.1$     & 91.4 \scriptsize	$\pm 0.1$  & 82.0 \scriptsize	$\pm 1.0$   & 78.6 \scriptsize	$\pm 1.3$  & 76.1 \scriptsize	$\pm 1.0$ \\
			                                              & Difference                                   & 14.4 \scriptsize	$\pm 5.2$            & <0.1 \scriptsize	$\pm 0.1$  & 0.1 \scriptsize	$\pm 0.1$   & 8.9 \scriptsize	$\pm 1.3$    & 13.4 \scriptsize	$\pm 0.9$  & 7.3  \scriptsize	$\pm 0.7$ \\
			                                              & \# Histograms                                & 595                                  & 25\,417                    & 26\,037                    & 3\,624                      & 3\,906                     & 3\,950                    \\
			\cmidrule{1-8}
			\multirow{4}{*}{2}                            & Train                                        & 93.5 \scriptsize	$\pm 2.1$            & 92.0 \scriptsize	$\pm 0.2$  & 91.9 \scriptsize	$\pm 0.3$  & 96.9 \scriptsize	$ \pm 1.9$  & 98.3 \scriptsize	$\pm 0.5$  & 91.1 \scriptsize	$\pm 0.5$ \\
			                                              & Test                                         & 62.7 \scriptsize	$\pm 7.2$            & 92.1 \scriptsize	$\pm 0.1$  & 91.0 \scriptsize	$\pm 0.6$  & 82.5 \scriptsize	$\pm 1.0$   & 80.5 \scriptsize	 $\pm 1.3$ & 78.1 \scriptsize	$\pm 1.5$ \\
			                                              & Difference                                   & 39.9 \scriptsize	$\pm 5.5$            & -0.1 \scriptsize	$\pm 0.2$  & 1.0 \scriptsize	$\pm 0.3$   & 14.4 \scriptsize	$\pm 1.0$   & 17.8 \scriptsize	$\pm 1.0$  & 13.0 \scriptsize	$\pm 1.5$ \\
			                                              & \# Histograms                                & 595                                  & 26\,872                    & 27\,353                    & 4\,239                      & 4\,027                     & 4\,055                    \\
			\cmidrule{1-8}
			\multirow{4}{*}{3}                            & Train                                        & 98.0  \scriptsize	$\pm 2.5$           & 92.1 \scriptsize	$\pm 0.3$  & 92.1 \scriptsize	$\pm 0.2$  & 99.4 \scriptsize $ \pm 0.9$ & 99.8 \scriptsize	$\pm 0.1$  & 93.6 \scriptsize	$\pm 1.2$ \\
			                                              & Test                                         & 58.7 \scriptsize	$\pm 5.3$            & 92.1 \scriptsize	$\pm 0.2$  & 91.5 \scriptsize	$\pm 0.2 $ & 82.8 \scriptsize	$\pm 1.0$   & 83.5 \scriptsize	$\pm 0.7$  & 77.8 \scriptsize	$\pm 1.8$ \\
			                                              & Difference                                   & 39.4 \scriptsize	$\pm 2.8$            & 0.1 \scriptsize	$\pm 0.2$   & 1.0 \scriptsize	$\pm 0.1$   & 16.6 \scriptsize	$\pm 1.0 $  & 16.3 \scriptsize	$\pm 0.7$  & 15.8 \scriptsize	$\pm 1.4$ \\
			                                              & \# Histograms                                & 595                                  & 27\,048                    & 27\,524                    & 4\,317                      & 4\,039                     & 4\,067                    \\
			\cmidrule{1-8}
			\multirow{4}{*}{4}                            & Train                                        & 99.8 \scriptsize	$\pm 0.3$            & 92.0 \scriptsize	$\pm 0.1$  & 92.2 \scriptsize	$\pm 0.2$  & 99.1 \scriptsize	$ \pm 0.2$  & 99.8 \scriptsize	$<0.1$     & 96.9 \scriptsize	$\pm 1.0$ \\
			                                              & Test                                         & 62.7 \scriptsize	$\pm 2.5$            & 92.1 \scriptsize	$\pm 0.1$  & 91.5 \scriptsize	$\pm 0.2$  & 82.7 \scriptsize	$\pm 0.8$   & 83.2 \scriptsize	$\pm 0.4$  & 79.8 \scriptsize	$\pm 1.2$ \\
			                                              & Difference                                   & 37.1 \scriptsize	$\pm 2.5$            & -0.1 \scriptsize	$\pm 0.1$  & 1.0 \scriptsize	$\pm 0.1$   & 16.4 \scriptsize	$\pm 0.7$   & 16.6 \scriptsize	$\pm 0.4$  & 17.2 \scriptsize	$\pm 0.8$ \\
			                                              & \# Histograms                                & 595                                  & 27\,059                    & \textsc{OOM}               & 4\,317                      & 4\,039                     & 4\,067                    \\
			\cmidrule{1-8}
			\multirow{4}{*}{5}                            & Train                                        & 98.9 \scriptsize	$\pm 1.9$            & 92.1 \scriptsize	$\pm 0.2$  & 92.4 \scriptsize	$\pm 0.2$  & 99.9 \scriptsize	$ \pm 0.2$  & 99.8 \scriptsize	$\pm 0.0$  & 97.7 \scriptsize	$\pm 0.9$ \\
			                                              & Test                                         & 57.0 \scriptsize	$\pm 3.9$            & 92.3 \scriptsize	$\pm 0.2$  & 91.6 \scriptsize	$\pm 0.2$  & 83.0 \scriptsize	$\pm 0.8$   & 84.1 \scriptsize	$\pm 1.1$  & 79.6 \scriptsize	$\pm 0.5$ \\
			                                              & Difference                                   & 41.9 \scriptsize	$\pm 2.9$            & -0.2 \scriptsize	$\pm 0.2$  & 1.0 \scriptsize	$\pm 0.2$   & 16.9 \scriptsize	$\pm 0.7$   & 15.7 \scriptsize	$\pm 1.1$  & 18.1 \scriptsize	$\pm 0.5$ \\
			                                              & \# Histograms                                & 595                                  & \textsc{OOM}               & \textsc{OOM}               & 4\,317                      & 4\,039                     & 4\,067                    \\
			\cmidrule{1-8}
			\multirow{4}{*}{6}                            & Train                                        & 99.4 \scriptsize	$\pm 0.8$            & 92.0 \scriptsize	$\pm 0.2$  & 92.2 \scriptsize	$\pm 0.2$  & 99.1 \scriptsize	$ \pm 1.9$  & 99.6 \scriptsize	$\pm 0.6$  & 95.2 \scriptsize	$\pm 1.9$ \\
			                                              & Test                                         & 54.0 \scriptsize	$\pm 2.3$            & 92.2 \scriptsize	$\pm 0.2$  & 91.4 \scriptsize	$\pm 0.4$  & 83.5 \scriptsize	$\pm 1.0$   & 83.4 \scriptsize	$\pm 1.3$  & 79.2 \scriptsize	$\pm 1.3$ \\
			                                              & Difference                                   & 44.4 \scriptsize	$\pm 1.9$            & -0.2 \scriptsize	$\pm 0.1$  & 1.0 \scriptsize	$\pm 0.2$   & 15.6 \scriptsize	$\pm 1.2$   & 16.2 \scriptsize	$\pm 0.9$  & 16.0 \scriptsize	$\pm 2.1$ \\
			                                              & \# Histograms                                & 595                                  & \textsc{OOM}               & \textsc{OOM}               & 4\,317                      & 4\,039                     & 4\,067                    \\
			\bottomrule
		\end{tabular}
	}
\end{table}

\subsection{Results and discussion}
In the following, we answer questions \textbf{Q1} to \textbf{Q3}.

\paragraph{Q1} See~\cref{fig:exp1} in the appendix). Increasing the feature dimension $d$ increases the average difference between train and test accuracies across all datasets. For example, on the \textsc{Enzymes} dataset, the difference increases from around 5\% for $d = 4$ to more than 45\% for $d = 1024$. However, we also observe that the difference does not increase when reaching near-perfect training accuracies, i.e., going from $d = 256$ to $d = 1\,024$ does not increase the difference. Hence, the results show that the number of parameters plays a crucial role in GNNs' generalization ability, in accordance with~\cref{thm:bartlett}.

\paragraph{Q2} See~\cref{colors}. The results indicate that the number of \wlone{}-distinguishable graphs ($m_{n,d,L}$) influence GNNs' generalization properties. For example, on the \textsc{Mutagenicity} dataset, after two iterations, the number of unique histograms computed by \wlone{} stabilizes, and similarly, the generalization error stabilizes as well. Similar effects can be observed for the \textsc{Enzymes}, \textsc{NCI1}, and \textsc{NCI109} datasets. Hence, our results largely confirm~\cref{thm:colorbound_up,prop:matchingvc}.

\paragraph{Q3}
See~\cref{fig:exp1}.  Increasing the feature dimension boosts the model's capacity to fit random class labels, indicating that increased bitlength implies an increased VC dimension. For example, for an order of 70, a GNN using a feature dimension of 4 cannot reach an accuracy of over 75\%. In contrast, feature dimensions 64 and 256 can almost fit such data. Moreover, for larger graphs, up to order 90, a GNN with a feature dimension of 256 can almost perfectly fit random class labels, with a feature dimension of 64 only slightly worse, confirming~\cref{thm:bl_lowerr}.

\section{Conclusion}
We investigated GNNs' generalization capabilities through the lens of VC dimension theory in different settings. Specifically, when not assuming a bound on the graphs' order, we showed that the VC dimension tightly depends on the bitlength of the GNNs' weights. We further showed that the number of colors computed by the \wlone{}, besides the number of parameters and layers, influences the VC dimension.
When a bound on the graphs' order is known, we upper and lower bounded GNNs' VC dimension via the maximal number of graphs distinguishable by the \wlone. \emph{Thus, our theory provides the first link between expressivity results and generalization.} Further, our theory also applies to a large set of recently proposed GNN enhancements.

\section*{Acknowledgements}
Christopher Morris is partially funded by a DFG Emmy Noether grant (468502433) and RWTH Junior Principal Investigator Fellowship under Germany's Excellence Strategy. Martin Grohe is partially funded by the European Union (ERC, SymSim,
101054974). Views and opinions expressed are, however, those of the author(s) only and do not necessarily reflect those of the European Union or the European Research Council. Neither the European Union nor the granting authority can be held responsible for them.

\setcitestyle{numbers}
\bibliography{bibliography}

\appendix

\section{Extended related work}\label{app:related_work}
In the following, we discuss more related work.

\paragraph{Expressive power of \kwl}
The Weisfeiler--Leman algorithm constitutes one of the earliest and most natural approaches to isomorphism testing~\citep{Wei+1976,Wei+1968}, and the theory community has heavily investigated it over the last few decades~\citep{Gro2017}. Moreover, the fundamental nature of the \kwl{} is evident from various connections to other fields such as logic, optimization, counting complexity, and quantum computing. The power and limitations of the \kwl{} can be neatly characterized in terms of logic and descriptive complexity~\citep{Bab1979,Imm+1990}, Sherali-Adams relaxations of the natural integer linear optimization problem for the graph isomorphism problem~\citep{Ast+2013,GroheO15,Mal2014}, homomorphism counts~\citep{Del+2018}, and quantum isomorphism games~\citep{Ats+2019}. In their seminal paper,~\citet{Cai+1992} showed that, for each $k$, a pair of non-isomorphic graphs of size $\cO(k)$ exists not distinguished by the \kwl. \citet{Kie2020a} gives a thorough survey of more background and related results concerning the expressive power of the \kwl. For $k=1$, the power of the algorithm has been completely characterized~\citep{Arv+2015,Kie+2015}. Moreover, upper bounds on the running time~\citep{Ber+2017a} and the number of iterations for $k=1$~\citep{Kie+2020} and the non-oblivious $k=2$~\citep{Kie+2016,Lic+2019} have been shown. For $k$ in $\{1,2\}$,~\citet{Arv+2019} studied the abilities of the (non-oblivious) \kwl{} to detect and count fixed subgraphs, extending the work of~\citet{Fue+2017}. The former was refined in~\citep{Che+2020a}. \citet{Kie+2019} showed that the non-oblivious  $3$-$\mathsf{WL}$ completely captures the structure of planar graphs. The algorithm for logarithmic $k$ plays a prominent role in the recent result of \cite{Bab+2016} improving the best-known running time for the graph isomorphism problem. Recently,~\citet{Gro+2020a} introduced the framework of Deep Weisfeiler--Leman algorithms, which allow the design of a more powerful graph isomorphism test than Weisfeiler--Leman type algorithms. Finally, the emerging connections between the Weisfeiler--Leman paradigm and graph learning are described in two recent surveys~\citep{Gro+2020,Mor+2022}.

\section{Oblivious \texorpdfstring{\kwl}{k-WL}}\label{kwl}

Intuitively, to surpass the limitations of the \wlone, the \kwl colors ordered subgraphs instead of a single vertex.\footnote{There exists two definitions of the \kwl, the so-called oblivious \kwl{} and the folklore or non-oblivious \kwl{}; see~\citet{Gro+2021}. There is a subtle difference in how they aggregate neighborhood information. Within the graph learning community, it is customary to abbreviate the oblivious \kwl{} as \kwl, a convention we follow in this paper.} More precisely, given a graph $G$, the \kwl colors the tuples from $V(G)^k$ for $k \geq 2$ instead of the vertices. By defining a neighborhood between these tuples, we can define a coloring similar to the \wlone. Formally, let $G$ be a graph, and let
$k \geq 2$. In each iteration, $t \geq 0$, the algorithm, similarly to the \wlone, computes a
\new{coloring} $C^k_t \colon V(G)^k \to \Nb$. In the first iteration, $t=0$, the tuples $\vec{v}$ and $\vec{w}$ in $V(G)^k$ get the same
color if they have the same atomic type, i.e., $C^k_{0}(\vec{v}) \coloneqq \text{atp}(\vec{v})$. Here, we define the atomic type $\text{atp} \colon V(G)^k \to \Nb$, for $k > 0$, such that $\text{atp}(\vec{v}) = \text{atp}(\vec{w})$ for $\vec{v}$ and $\vec{w}$ in $V(G)^k$ if and only if the mapping $\varphi\colon V(G)^k \to V(G)^k$ where $v_i \mapsto w_i$ induces a partial isomorphism, i.e., we have $v_i = v_j \iff w_i = w_j$ and $(v_i,v_j) \in E(G) \iff (\varphi(v_i),\varphi(v_j)) \in E(G)$. Then, for each layer, $t > 0$, $C^k_{t}$ is defined by
\begin{equation*}\label{ci}
	C^k_{t}(\vec{v}) \coloneqq \REL \big(C^k_{t-1}(\vec{v}), M_t(\vec{v}) \big),
\end{equation*}
with $M_t(\vec{v})$ the multiset
\begin{equation*}\label{mi}
	M_t(\vec{v}) \coloneqq  \big( \{\!\! \{  C^{k}_{t-1}(\phi_1(\vec{v},w)) \mid w \in V(G) \} \!\!\}, \dots, \{\!\! \{  C^{k}_{t-1}(\phi_k(\vec{v},w)) \mid w \in V(G) \} \!\!\} \big),
\end{equation*}
and where
\begin{equation*}
	\phi_j(\vec{v},w)\coloneqq (v_1, \dots, v_{j-1}, w, v_{j+1}, \dots, v_k).
\end{equation*}
That is, $\phi_j(\vec{v},w)$ replaces the $j$-th component of the tuple $\vec{v}$ with the vertex $w$. Hence, two tuples are \new{adjacent} or \new{$j$-neighbors} if they are different in the $j$th component (or equal, in the case of self-loops). Hence, two tuples $\vec{v}$ and $\vec{w}$ with the same color in iteration $(t-1)$ get different colors in iteration $t$ if there exists a $j$ in $[k]$ such that the number of $j$-neighbors of $\vec{v}$ and $\vec{w}$, respectively, colored with a certain color is different.

We run the \kwl algorithm until convergence, i.e., until for $t$ in $\Nb$
\begin{equation*}
	C^k_{t}(\vec{v}) = C^k_{t}(\vec{w}) \iff C^k_{t+1}(\vec{v}) = C^k_{t+1}(\vec{w}),
\end{equation*}
for all $\vec{v}$ and $\vec{w}$ in $V(G)^k$, holds. For such $t$, we define
$C^k_{\infty}(\vec{v}) = C^k_t(\vec{v})$ for $\vec{v}$ in $V(G)^k$. At convergence, we call the partition of $V(G)^k$ induced by $C^k_t$ the \new{stable partition}. We set $C_\infty^k(v)\coloneqq C_\infty^k(v,\ldots,v)$ and refer to this as the color of the vertex $v$.

Similarly to the \wlone, to test whether two graphs $G$ and $H$ are non-isomorphic, we run the \kwl in ``parallel'' on both graphs. Then, if the two graphs have a different number of vertices colored $c$, for $c$ in $\Nb$, the \kwl{} \textit{distinguishes} the graphs as non-isomorphic. By increasing $k$, the algorithm gets more powerful in distinguishing non-isomorphic graphs, i.e., for each $k \geq 2$, there are non-isomorphic graphs distinguished by $(k+1)$\text{-}\textsf{WL} but not by \kwl~\citep{Cai+1992}. For a finite set of graphs $S \subset \cG$, we run the algorithm in ``parallel''  over all graphs in the set $S$.

\section{\texorpdfstring{$k$}{k}-order GNNs}\label{kgnn}

By generalizing~\cref{def:gnn} in~\cref{sec:gnn}, following~\cite{Mor+2019,Morris2020b,Mor+2022}, we can derive $k$-GNNs computing features for all $k$-tuples $V(G)^k$, for $k > 0$, defined over the set of vertices of an attributed graph $G=(V(G),E(G),a)$ with features from $\Rb^d$. Concretely, in each layer, $t > 0$, for each $k$-tuple $\vec{v}=(v_1,\ldots,v_k)$ in $V(G)^k$, we compute a feature
\begin{align*}
	\begin{split}
		\hb_{\vec{v}}^\tup{t} \coloneqq \UPD^\tup{t} \Big( \hb_{\vec{v}}^\tup{t-1} ,\AGG^\tup{t} \big(&\oms \hb_{\vec{v}}^\tup{t-1}(\phi_1(\vec{v},w)) \mid w \in V(G) \cms, \dots,\\ &\oms \hb_{\vec{v}}^\tup{t-1}(\phi_k(\vec{v},w)) \mid w \in V(G) \cms \big) \!\Big).
	\end{split}
\end{align*}
Initially, for $t = 0$, we set
\begin{align*}
	\hb_{\vec{v}}^\tup{0} \coloneqq \UPD([\text{atp}(\vec{v}), a(v_1), \dots, a(v_k)]) \in \Rb^d,
\end{align*}
i.e., the atomic type and the attributes of a given $k$-tuple determine the initial feature of a $k$-tuple's vertices. In the above, $\UPD$, $\UPD^\tup{t}$, and $\AGG^\tup{t}$ may be differentiable parameterized functions, e.g., neural networks. In the case of graph-level tasks, e.g., graph classification, one additionally uses
\begin{equation*}\label{k-readout}
	\hb_G \coloneqq \RO\bigl( \oms \hb_{\vec{v}}^{\tup{L}}\mid \vec{v}=(v,\dots,v), v\in V(G) \cms \bigr) \in \Rb^{d},
\end{equation*}
to compute a single vectorial graph representation based on the learned $k$-tuple features after iteration $L$.

\subsection{Transfering VC bounds from GNNs to \texorpdfstring{$k$}{k}-order GNNs and other more expressive architectures}\label{klift_app}
In the following, we briefly sketch how~\cref{thm:colorbound_up,prop:matchingvc,infty_simple,thm:bartlett} can be lifted to $k$-GNNs. First, observe that we can simulate the computation of a $k$-GNN via a GNN on a sufficiently defined auxiliary graph. That is, the auxiliary graph contains a vertex for each $k$-tuple, and an edge connects two $k$-tuples $j$ if they are $j$-neighbors for $j$ in $[k]$; see~\citet{Mor+2022} for details. Using a $1$-WL equivalent GNN taking edge labels into account, we can extend \cref{thm:colorbound_up,prop:matchingvc,infty_simple} to $k$-GNNs. Similar reasoning applies to~\cref{thm:bartlett}, i.e., we can apply the proof technique from~\cref{proof_bartlett_app} to this auxiliary graph.

\subsubsection{Architectures based on subgraph information} Further, we note that~\cref{thm:colorbound_up,prop:matchingvc} also easily extend to recent GNN enhancements, e.g., subgraph-based~\citep{botsas2020improving} or subgraph-enhanced GNNs~\citep{Bev+2021,Qia+2022}. Since suitably defined variations of the \wlone, incorporating subgraph information at initialization, upper bound the architectures' expressive power, we can easily apply the reasoning behind the proofs of~\cref{thm:colorbound_up,prop:matchingvc} to these cases. Hence, the architectures' VC dimensions are also tightly related to the number of graphs distinguishable by respective \wlone{} variants.

\section{Relationship between VC dimension and generalization error}\label{vc_error}
If we can bound the VC dimension of a hypothesis class $\cC$ of GNNs, we directly get insights into its generalization ability, i.e., the difference 
of the empirical error $R_S(h)$ and the true error $R_{\cD}(h)$ for $h \in \cC$ and a data generating distribution $\cD$.
\begin{theorem}
Let $\cC$ be a class of GNNs, with finite VC dimension $VC(\cC) = d$. Then for $\cC$, for all $\varepsilon > 0$ and $\delta \in (0,1)$, using
    \begin{equation*}
        m = \cO \left( \frac{1}{\varepsilon^2} \left( d \ln \left( \frac{d}{\varepsilon} \right) + \ln \left( \frac{1}{\delta} +1 \right) \right)  \right)
    \end{equation*}
samples, for all data generating distributions $\cD$, we have 
    \begin{equation*}
        \Pr_{S \simeq \cD^m} \left( \forall h \in \cC: \left| R_S(h) - R_{\cD}(h) \right| \leq \varepsilon \right) \geq 1 - \delta.
    \end{equation*}
\end{theorem}
This result was first proven by Vladimir Vapnik and Alexey Chervonenkis in 1960's; see, e.g.,~\citet{Moh+2012} for a proof.

\section{Simple GNNs}\label{sec:simpleGNNs}
We here provide more detail on the simple GNNs mentioned in~\cref{sec:gnn}. That is,
for given $d$ and $L$ in $\Nb$, we define the class $\GNN_{\mathsf{mlp}}(d,L)$ of simple GNNs as $L$-layer
GNNs for which, according to~\cref{def:gnn}, for each $t$  in $[L]$, the aggregation function $\AGG^\tup{t}$ is simply
summation and the update function
$\UPD^\tup{t}$ is a multilayer perceptron $\mathsf{mlp}^\tup{t}:\Rb^{2d}\to\Rb^d$ of width at most $d$. Similarly, the readout function in \cref{def:readout}
consists of a multilayer perceptron $\mathsf{mlp}:\Rb^d\to\Rb$ applied on the sum of all vertex features computed in layer $L$.\footnote{For simplicity we assume that all feature dimensions of the layers are fixed to $d$ in $\Nb$ and also assume that the readout layer returns a scalar.} More specifically,
GNNs in $\GNN_{\mathsf{mlp}}(d,L)$ compute on a graph $(G,\mL)$ in $\cG_{d}$, for each $v\in V(G)$,
\begin{equation}\label{def:sgnn}
	\hb_{v}^\tup{t} \coloneqq
	\mathsf{mlp}^\tup{t}\Bigl(\hb_{v}^\tup{t-1},\sum_{u\in N(v)}\hb_{u}^\tup{t-1}\Bigr) \in \Rb^{d},
\end{equation}
for $t$ in $[L]$ and $\hb_v^\tup{0}\coloneqq\mL_{v.}$, and
\begin{equation}\label{def:sreadout}
	\hb_G \coloneqq \mathsf{mlp}\Bigl(\sum_{v\in V(G)}\hb_{v}^{\tup{L}}\Bigr) \in \Rb.
\end{equation}

We also consider an even simpler class $\GNN_{\mathsf{slp}}(d,L)$ of $\GNN_{\mathsf{mlp}}(d,L)$ in which the multilayer perceptrons
are in fact single layer perceptrons. That is, \cref{def:sgnn} is replaced by
\begin{equation}\label{def:ssgnn}
	\hb_{v}^\tup{t} \coloneqq
	\sigma_t\Bigl(\hb_{v}^\tup{t-1}\mW_1^{(t)}+\sum_{u\in N(v)}\hb_{u}^\tup{t-1}\mW_2^{(t)}+\mb^{(t)}\Bigr) \in \Rb^{d},
\end{equation}
where $\mW_1^{(t)}$ in $\Rb^{d\times d}$ and $\mW_2^{(t)}$ in $\Rb^{d\times d}$ are weight matrices, and
$\mb^{(t)}$ in $\Rb^{1\times d}$ is a bias vector, and $\sigma_t \colon \Rb\to\Rb$ is an activation function, for $t$ in $[L]$.
Similarly, \cref{def:sreadout} is replaced by
\begin{equation}\label{def:ssreadout}
	\hb_G \coloneqq \sigma_{L+1}\Bigl(\sum_{v\in V(G)}\hb_{v}^{\tup{L}}\mw+b\Bigr) \in \Rb.
\end{equation}
with $\mw$ in $\Rb^{d\times 1}$  a weight vector and $b$ in $\Rb$  a  bias value of the final readout layer. Also,
$\sigma_{L+1} \colon \Rb\to\Rb$ is an activation function. We can thus represent elements in $\GNN_{\mathsf{slp}}(d,L)$
more succinctly by the following tuple of parameters,
\begin{equation*}
	\Theta=\Big(\mW_1^{(1)},\mW_2^{(1)}, \mb^{(1)},\ldots,\mW_1^{(L)},\mW_2^{(L)},\mb^{(L)},\mw,b \Big),
\end{equation*}
together with the tuple of activation functions $\pmb\sigma=(\sigma_1,\ldots,\sigma_L,\sigma_{L+1})$. We can equivalently view $\Theta$ as an element in $\Rb^{\nw}$.
Each $\Theta$ in $\Rb^{\nw}$ and $\pmb\sigma=(\sigma_1,\ldots,\sigma_{L+1})$ induces a permutation-invariant \textit{graph function}
\begin{equation*}
	\gnn_{\Theta,\pmb\sigma} \colon \cG_{d}\to\Rb \colon
	(G,\mL)\mapsto \gnn_{\Theta,\pmb\sigma}(G,\mL)\coloneqq\hb_G,
\end{equation*}
with $\hb_G$ as defined in \cref{def:ssreadout}.

\section{Missing proofs}\label{sec:appendix}

In the following, we outline missing proofs from the main paper.

\subsection{Proofs of \cref{thm:colorbound_up} and \cref{prop:matchingvc}}
We start with the general upper bound on the VC dimension in terms of the number of \wlone-indistinguishable graphs.
\begin{proposition}[\cref{thm:colorbound_up} in the main text]
	For all $n$, $d$, and $L$,  the maximal number of graphs of  order at most $n$ with $d$-dimensional boolean features that can be shattered by $L$-layer GNNs is bounded by the maximal number ($m_{n,d,L}$) of \wlone-distinguishable graphs. That is,
	\begin{equation*}
		\vcdim_{\cG_{d,n}^\bool}\!\!\bigl(\GNN(L)\bigr)\leq m_{n,d,L}.
	\end{equation*}
\end{proposition}
\begin{proof}
	Clearly, every set $\cS$ of $m_{n,\inid,L}+1$ graphs from $\cG_{\inid,n}^\bool$ contains at least two graphs $\mG$ and $\mG'$ not distinguishable by the \wlone{}. Since GNNs cannot distinguish \wlone-indistinguishable graphs~\cite{Mor+2019,Xu+2018b}, they cannot tell $\mG$ and $\mG'$ apart and hence cannot not shatter $\cS$. Hence, the VC dimension can be at most $m_{n,\inid,L}$.
\end{proof}

We next show a corresponding lower bound. In fact, the lower bound already holds for the class of
simple GNNs of arbitrary width, that is for GNNs in $\GNN_{\mathsf{mlp}}(L)\coloneqq\bigcup_{d\in\Nb}\GNN_{\mathsf{mlp}}(d,L)$.
\begin{proposition}[\cref{prop:matchingvc} in the main paper]
	For all $n$, $d$, and $L$,  all $m_{n,d,L}$ \wlone-distinguishable graphs of order at most $n$ with $d$-dimensional boolean features can be shattered by sufficiently
	wide $L$-layer GNNs. Hence,
	\begin{equation*}
		\vcdim_{\cG_{d,n}^\bool}\!\!\bigl(\GNN(L)\bigr)=m_{n,d,L}.
	\end{equation*}
\end{proposition}
\begin{proof}
	For all $i$ in $[m_{n,\inid,L}]$, choose $\mG_i$ in  $\cG_{\inid,n}^\bool$ such that
	$\cS=\{\mG_1,\ldots,\mG_{m_{n,\inid,L}}\}$ consists of the maximum number of graphs in $\cG_{\inid,n}^\bool$
	pairwise distinguishable by the \wlone{} after $L$ iterations.

	We next show that the class of  simple GNNs which are wide enough, that is, $\GNN_{\mathsf{mlp}}(d',L)$ for large enough $d'$, is sufficiently rich to shatter $\cS$. That is, we show that
	for each $\cT\subseteq\cS$ there is a $\gnn_{\cT}$ in $\GNN_{\mathsf{mlp}}(d',L)$
	such that for all $i$ in $[m_{n,d,L}]$:
	\begin{equation*}
		\gnn_{\cT}(\mG_i)=\begin{cases}
			1 & \text{if $\mG_i\in\cT$, and} \\
			0 & \text{otherwise}.
		\end{cases}
	\end{equation*}
	This shows that $\cS$ is shattered by $\GNN_{\mathsf{mlp}}(d',L)$
	and hence its VC dimension is at least $|\cS|=m_{n,d,L}$, as desired.

	\textbf{Overview of the construction} Intuitively, we will show that $\GNN_{\mathsf{mlp}}(d',L)$,
	with $d'$ large enough, is powerful enough to return a one-hot encoding  of the color histograms of graphs in $\cS$. That is, there is a simple GNN $\gnn$ in $\GNN_{\mathsf{mlp}}(d',L)$ which in the MLP in its readout layer embeds a graph $\mG_i$ in $\cS$ as a vector $\mathbf{h}_{\mG}$ in $\{0,1\}^{m_{n,d,L}}$ satisfying $(h_{\mG})_i=1$
	if and only if $\mG_i$ in $\cT$ and $i\in[m_{n,d,L}]$. Then, we extend the readout multilayer perceptron of $\gnn$  by one more layer such that on input $\mG$ the revised GNN evaluates to the scalar
	\begin{eqnarray*}
		g_{\mG}\coloneqq\mathsf{sign}(\mathbf{h}_{\mG}\cdot \mathbf{w}^{\textsc{t}} - 1 )\in \{ 0,1\},
	\end{eqnarray*}
	with $\mathbf{w}$ in $\Rb^{d'\times 1}$.
	We observe that given $\cT\subseteq\cS$ it  suffices to let the parameter vector
	$\mathbf{w}$ be the indicator vector for $\cT$. Indeed, this ensures that $g_\mG=1$ if and only if $\mG$ is in a color class included in $\cT$. We can explore all such subsets $\cT$ of $\cS$ by varying $\mathbf{w}$; hence, this GNN will shatter $\cS$.

	\textbf{Encoding \wlone{} colors via GNNs} We proceed with the construction of the required GNN. For simplicity of exposition, in the description below we will construct GNN layers of non-uniform width. One can easily obtain uniform width by padding each layer.
	First, by~\citet[Theorem 2]{Mor+2019}, there exists a GNN architecture with feature dimension (at most) $n$ and consisting of $L$ layers such that for each $\mG_i$ in $\cS$ it computes \wlone-equivalent vertex features $\fb_v$ in $\Rb^{1 \times d}$ for $v\in V(G_i)$.
	That is, for vertices $v$ and $w$ in $V(G_i)$ it holds that
	\begin{equation*}
		\fb_{v}=  \fb_{w} \iff C^1_{L}(v) = C^1_{L}(w).
	\end{equation*}
	We note here that we can construct a single GNN architecture for all graphs by applying~\citep[Theorem 2]{Mor+2019} over the disjoint union the graphs in $\cS$. This increases the width from $n$ to $nm_{n,d,L}$.
 
	\textbf{Encoding \wlone{} histograms via GNNs} Moreover, again by~\citep[Theorem 2]{Morris2020b} there exists $\mW$ in $\Rb^{nm_{n,d,L} \times nm_{n,d,L}}$ and $\vec{b}$ in $\Rb^{nm_{n,d,L}}$ such that
	\begin{equation*}
		\sigma\Big(\sum_{v \in V(G)} \fb_{v}\mW  + \vec{b} \Big) =  \sigma \Big(	\sum_{v \in V(H)}  \fb_{v}\mW  + \vec{b}  \Big) \iff h_{\mG} = h_{\mH},
	\end{equation*}
	for graphs $\mG$ and $\mH$ in $\cS$. We use $\mathsf{ReLU}$ as activation function
	$\sigma$ here, just as in~\cite{Mor+2019}. Other activation functions could be used as well~\citep{Gro+2021}. Hence, for each graph in $\cS$, we have a vector in $\Rb^{1 \times nm_{n,d,L}}$ uniquely encoding it. Since the number of vertices $n$ is fixed, there exists a number $M$ in $\Nbb$ such that $M\sigma(\sum_{v \in V(G)} \mW  \fb_{v})$ is in $\Nb^{1 \times nm_{n,d,L}}$ for all $\mG$ in $\cS$. Moreover, observe that there exists a matrix $\mW'$ in  $\Nb^{nm_{n,d,L} \times 2m_{n,d,L}}$ such that
	\begin{equation*}
		M \sigma\Big(\sum_{v \in V(G)}   \fb_{v} \mW + \vec{b} \Big) \mW' =  	M\sigma \Big(\sum_{v \in V(H)} \fb_{v} \mW + \vec{b}\Big) \mW' \iff h_{\mG} = h_{\mH},
	\end{equation*}
	for graphs $\mG$ and $\mH$ in $\cS$. For example, we can set
	\begin{equation*}
		\mW' = \begin{bmatrix}
			K^{nm_{n,L}-1} & \cdots & K^{nm_{n,L}-1} \\
			\vdots         & \cdots & \vdots         \\
			K^{0}          & \cdots & K^{0}
		\end{bmatrix}
		\in \Nbb^{nm_{n,d,L} \times 2m_{n,d,L}}
	\end{equation*}
	for sufficiently large $K > 1$. Hence, the above GNN architecture computes a vector $\mathbf{k}_{\mG}$ in $\Nbb^{2m_{n,d,L}}$ containing $2m_{n,d,L}$ occurrences  of a natural number uniquely encoding each color histogram for each graph $\mG$ in $\cS$.

	We next turn  $\mathbf{k}_{\mG}$ into our desired $\mathbf{h}_{\mG}$ as follows.
	We first define an intermediate vector $\hb_{\mG}'$ whose entries will be used to
	check which color histogram is returned.
	More specifically, we define
	\begin{equation*}
		\mathbf{h}_{\mG}'=\mathsf{lsig}(\mathbf{k}_{\mG}\cdot (\mathbf{w}'')^{\textsc{t}}+\mathbf{b}),
	\end{equation*}
	with $\mathbf{w}''=(1,-1,1,-1,\ldots,1,-1)\in\Rb^{2m_{n,L}}$ and $\mathbf{b}=(-c_1-1,c_1+1,-c_2-1,c_2+1,\ldots,-c_{m_{n,d,L}}-1,c_{m_{n,d,L}}+1)\in\Rb^{2m_{n,d,L}}$ with $c_i$ the number encoding the $i$th color histogram.  We note that for odd $i$,
	\begin{equation*}
		(h_{\mG}')_i\coloneqq \mathsf{lsig}(\mathsf{col}(G)-c_i-1)=\begin{cases}
			1 & \mathsf{col}(G)\geq c_i \\
			0 & \text{otherwise}.
		\end{cases}
	\end{equation*}
	and for even $i$,
	\begin{equation*}
		(h_{\mG}')_i\coloneqq \mathsf{lsig}(-\mathsf{col}(G)+c_i+1)=\begin{cases}
			1 & \mathsf{col}(G)\leq c_i \\
			0 & \text{otherwise}.
		\end{cases}
	\end{equation*}
	In other words, $((h_{\mG}')_i,(h_{\mG}')_{i+1})$ are both $1$ if and only if $\mathsf{col}(G)=c_i$. We thus obtain $\hb_{\mG}$ by
	combining $((h_{\mG}')_i,(h_{\mG}')_{i+1})$ using an ``AND'' encoding (e.g., $\mathsf{lsig}(x+y-1)$) applied to pairs of consecutive entries in $\mathbf{h}_\mG'$.
	That is,
	\begin{equation*}
		\hb_{\mG}\coloneqq\mathsf{lsig}\left(
		\hb_{\mG}'\cdot\begin{pmatrix}
			1      & 0      & \cdots & 0      \\
			1      & 0      & \cdots & 0      \\
			0      & 1      & \cdots & 0      \\
			0      & 1      & \cdots & 0      \\
			\vdots & \vdots & \ddots & \vdots \\
			0      & 0      & \cdots & 1      \\
			0      & 0      & \cdots & 1
		\end{pmatrix}- (1, 1,\ldots,1)
		\right)\in\Rb^{m_{d,n,L}}
	\end{equation*}
	We thus see that a $3$-layer MLP suffices for the readout layer of the simple GNN, finishing the proof. We remark that the maximal width is $2nm_{n,d,L}$, so we can take $d'=2nm_{n,d,L}$.
\end{proof}

\subsection{Proof of~\cref{thm:bl_lowerr}}

We now prove~\cref{thm:bl_lowerr}.

\begin{proposition}[\cref{thm:bl_lowerr} in the main text]
	There exists a family $\cF_b$ of simple  $2$-layer GNNs of width two and of bitlength $\cO(b)$ using a piece-wise linear activation such that its VC dimension is \emph{exactly} $b$.
\end{proposition}
\begin{proof}
	We first show the lower bound. We fix some $n\ge 1$. We shall construct a family of GNNs whose
	weights have bitlength $O(n)$ and a family of $n$ graphs shattered by
	these GNNs. Thereto, for all $\vec x=(x_1,\ldots,x_n)$ in $\{0,1\}^{n}$, we let
	\begin{equation*}
		\rho(\vec x) \coloneqq \sum_{i= 1}^n(2^{-2i+1}+x_i2 ^{-2i}).
	\end{equation*}
	Written in binary, we have
	\begin{equation*}
		\rho(\vec x)=0.1x_11x_21x_3\ldots 1x_n.
	\end{equation*}
	Observe that
	\begin{equation}
		\label{eq:1}
		\frac{1}{2}\le \rho(\vec x)\le 1.
	\end{equation}
	For $1\le k\le n$, we let
	\begin{equation*}
		\rho_k(\vec x)\coloneqq \rho\big((x_{k+1},\ldots,x_{k+n})\big)=\sum_{i=
			1}^n(2^{-2i+1}+x_{k+i}2 ^{-2i}),
	\end{equation*}
	where $x_{k+i}\coloneqq0$ for $k+i>n$.
	Then it follows from \eqref{eq:1} that
	\begin{equation}
		\label{eq:2}
		\frac{1}{2}\le \rho_k(\vec x)\le 1.
	\end{equation}
	We claim that
	\begin{equation}
		\label{eq:5}
		\rho_k(\vec x)=2^{2k}\rho(\vec
		x)-\underbrace{\left(\sum_{i=1}^{k}2^{2(k-i)+1}-\sum_{i=1}^{k}2^{2(k-n-i)+1}\right)}_{\eqqcolon
		{a_k}}-\underbrace{\sum_{i=1}^{k-1}2^{2(k-i)}x_i}_{\coloneqq
		{b_k}(\vec x)}-x_k
	\end{equation}
	Indeed, we have
	\begin{align*}
		2^{2k}\rho(\vec x) & =\sum_{i=1}^n\big(2^{2(k-i)+1}+x_i2^{2k-2i}\big)                                                                                              \\
		                   & =\sum_{i=1}^{n+k}\big(2^{2(k-i)+1}+x_i2^{2(k-i)}\big)-\sum_{i=n+1}^{n+k}2^{2(k-i)+1}                                                          \\
		                   & =\sum_{i=1}^{k}2^{2(k-i)+1}-\sum_{i=n+1}^{n+k}2^{2(k-i)+1}+\sum_{i=1}^{k}x_i2^{2(k-i)}+\sum_{i=k+1}^{n+k}\big(2^{2(k-i)+1}+x_i2^{2(k-i)}\big) \\
		                   & =\sum_{i=1}^{k}2^{2(k-i)+1}-\sum_{i=1}^{k}2^{2(k-n-i)+1}+\sum_{i=1}^{k}x_i2^{2(k-i)}+\sum_{i=1}^{n}\big(2^{-2i+1}+x_{k+i}2^{-2i}\big)         \\
		                   & =\sum_{i=1}^{k}2^{2(k-i)+1}-\sum_{i=1}^{k}2^{2(k-n-i)+1}+\sum_{i=1}^{k-1}x_i2^{2(k-i)}+x_k+\rho_k(\vec
		x)                                                                                                                                                                 \\
		                   & =a_k+b_k(\vec x)+x_k+\rho_k(\vec
		x),
	\end{align*}
	which proves \eqref{eq:5}.
	Now let
	\begin{equation*}
		c_k(\vec x)\coloneqq b_k(\vec x)+a_k+1.
	\end{equation*}
	Then by \eqref{eq:2} and \eqref{eq:5}, we have
	\begin{equation}
		\label{eq:6}
		x_k-\frac{1}{2}\le4^k\rho(\vec x)-c_k(\vec x)\le x_k.
	\end{equation}
	For $\vec x=(x_1,\ldots,x_n)$ and $\vec y=(y_1,\ldots,y_n)$ in $\{0,1\}^{n}$, we write $\vec x\neq_k\vec y$ if $x_i\neq
		y_i$ for some $i<k$. Observe that $\vec x\neq_k \vec y$ implies
	$\big|{b_k}(\vec x)-{b_k}(\vec y)\big|\ge 4$ and thus
	\begin{equation}
		\label{eq:7}
		\big|{c_k}(\vec x)-{c_k}(\vec y)\big|\ge 4.
	\end{equation}
	Let $A: \Rb \to \Rb$ be the continuous piecewise-linear function defined by
	\begin{equation*}
		A(x)\coloneqq
		\begin{cases}
			0    & \text{if }x<0,                \\
			2x   & \text{if }0\le x<\frac{1}{2}, \\
			1    & \text{if }\frac{1}{2}\le x<1, \\
			3-2x & \text{if }1\le x<\frac{3}{2}  \\
			0    & \text{if }\frac{3}{2}\le x.
		\end{cases}
	\end{equation*}
	Since $x_k\in\{0,1\}$, by \eqref{eq:6} we have
	\begin{equation}
		\label{eq:9}
		x_k=A\big(4^k\rho(\vec x)-{c_k}(\vec x)\big).
	\end{equation}
	If follows from \eqref{eq:7} that for $\vec y$ with $\vec y\neq_k\vec
		x$ we have
	\begin{equation}
		\label{eq:10}
		A\big(4^k\rho(\vec x)-{c_k}(\vec y)\big)=0.
	\end{equation}
	Let
	\begin{equation*}
		\cC_k\coloneqq\Big\{{c_k}(\vec y)\Bigmid\vec
		y\in\{0,1\}^{n}\Big\}.
	\end{equation*}
	Then
	\begin{equation}
		\label{eq:8}
		x_k=\sum_{c\in \cC_k}A\big(4^k\rho(\vec x)-c\big).
	\end{equation}
	Note that the only dependence on $\vec x$ of the right-hand side of \eqref{eq:8}
	is in $\rho(\vec x)$, because $\cC_k$
	does not depend on $\vec x$.

	Observe that $|\cC_k|=2^{k-1}$, because $c_k(\vec y)$ only depends on
	$y_1,\ldots,y_{k-1}\in\{0,1\}$ and is distinct for distinct values of
	the $y_i$.
	We have
	\begin{equation*}
		{a_k}=\sum_{i=1}^{k}2^{2(k-i)+1}-\underbrace{\sum_{i=1}^{k}2^{2(k-n-i)+1}}_{\eqqcolon
		s\le
		1}=2\sum_{i=0}^{k-1}4^i-s=\frac{2}{3}\big(4^{k}-1\big)-s.
	\end{equation*}
	Thus
	\begin{equation*}
		\frac{2}{3}\big(4^{k}-1\big)-1\le a_k\le
		\frac{2}{3}\big(4^{k}-1\big).
	\end{equation*}
	Furthermore,
	\begin{equation*}
		0\le{b_k}(\vec x)\le
		\sum_{i=1}^{k-1}2^{2(k-i)}=\sum_{i=1}^{k-1}4^{k-i}=4\sum_{i=0}^{k-2}4^i=\frac{4}{3}\big(4^{k-1}-1\big).
	\end{equation*}
	Thus
	\begin{equation}
		\label{eq:11}
		\frac{2}{3}\big(4^{k}-1\big)\le c\le \frac{2}{3}\big(4^{k}-1\big)+
		\frac{4}{3}\big(4^{k-1}-1\big)+1=4^k-1.
	\end{equation}

	Now for each $\rho$  in $\Rb$ we construct a $2$-layer GNN
	$\mathfrak{G}_{\rho}$ as follows:
	\begin{itemize}
		\item Initially, all nodes $v$ carry the $1$-dimensional feature $\hb_{v}^\tup{0} \coloneqq 1$.
		\item The first layer computes the 2-dimensional feature
		      $\Mat{h_{v,1}^\tup{1}\\ h_{v,2}^\tup{1}}$ defined by
		      \begin{align*}
			      h_{v,1}^\tup{1} & \coloneqq\sum_{w\in N(v)}\rho\cdot \hb_{w}^\tup{0}-\rho, \\
			      h_{v,2}^\tup{1} & \coloneqq\sum_{w\in N(v)}\hb_{w}^\tup{0}-1.
		      \end{align*}
		\item The second layer computes the 1-dimensional feature
		      $\hb_{v}^\tup{2}$ defined by
		      \begin{align*}
			      \hb_{v}^\tup{2} & =A\left(h_{v,1}^\tup{1} -\sum_{w\in
				      N(v)} h_{w,2}^\tup{1} \right).
		      \end{align*}
		\item The readout functions just takes the sum of all the $\hb_{v}^\tup{2}$.
	\end{itemize}
	We define a graph $F_k$ as follows. The graph $F_k$ is a forest of height $2$.
	\begin{itemize}
		\item $F_k$ has a root node $r_c$ for every $c\in \cC_k$.
		\item Each $r_c$ has a child $s_c$ and $4^k$ additional children $t_{c,1},\ldots,t_{c,4^k}$.
		\item The $t_{c,i}$ are leaves.
		\item Each $s_c$ has children $u_{c,1},\ldots,u_{c,c}$.
		\item The $u_{c,i}$ are leaves.
	\end{itemize}

	Now we run the GNN $\mathfrak{G}_{\rho}$ on $F_k$ with $\rho=\rho(\vec x)$ for
	some $\vec x=(x_1,\ldots,x_n)$ in $\{0,1\}^{n}$.
	\begin{itemize}
		\item We have $\hb_{v}^\tup{0} =1$ for all $v$ in $V(F_k)$.
		\item We have
		      \begin{align*}
			      \hb_{t_{c,i},1}^\tup{1} & = \hb_{u_{c,i},1}^\tup{1},  \\
			      \hb_{s_c,1 }^\tup{1}    & =c\rho,                     \\
			      \hb_{r_c,1 }^\tup{1}    & =4^k\rho,                   \\
			      \intertext{and}
			      \hb_{t_{c,i},2}^\tup{1} & =\hb_{u_{c,i},1}^\tup{1}=0, \\
			      \hb_{s_c,2}^\tup{1}     & =c,                         \\
			      \hb_{r_c,2}^\tup{1}     & =4^k.
		      \end{align*}
		\item We have
		      \begin{align*}
			      \hb_{t_{c,i}}^\tup{2}            & =A\big(-4^k\big)=0      \\
			      \hb_{u_{c,i}}^\tup{2}            & =A\big(-c  \big)=0      \\
			      \hb_{s_{c}}^\tup{2}              & =A\big(c\rho-4^k\big)=0 \\
			      \hb_{r_{c}}^\tup{2}              & =A\big(4^k\rho-c\big)
			      =
			      \begin{cases}
				      x_k & \text{if }c=c_k(\vec x) \\
				      0   & \text{otherwise}
			      \end{cases} & \text{by \eqref{eq:9} and \eqref{eq:10}.}
		      \end{align*}
		      To see that the first three equalities hold, recall that $A(x)\neq 0$ only if
		      $0<x<\frac{3}{2}$. Thus $A(-4^k)=0$. Moreover, by \eqref{eq:11} we
		      have $2\le c$ and thus $A(c)=0$. Finally, $A\big(c\rho-4^k\big)=0$
		      because $\rho< 1$ and $c\le 4^k-1$ by \eqref{eq:11} and therefore
		      $c\rho-4^k<0$.
		\item As there is exactly one node $r_c$ with $c=c_k(\vec x)$, the
		      readout is $\sum_{v\in V(F_k)} \hb_{v}^\tup{2}   =x_k$.
	\end{itemize}
	Hence
	\begin{equation*}
		\mathfrak{G}_{\rho(\vec x)}(F_k)=x_i
	\end{equation*}
	Thus the GNNs $\mathfrak{G}_{\rho(\vec x)}$ for $\vec x\in\{0,1\}^n$ shatter the set $\{F_1,\ldots,F_n\}$. Since the bitlength is upper bounded by $\cO(b)$ and the number of parameters in the above construction is constant, the hypothesis set is finite, and the upper bound follows from standard learning-theoretic results; see, e.g.,~\cite{Moh+2012}.
\end{proof}
\subsection{Proof of~\cref{thm:bartlett}}\label{proof_bartlett_app}
In the following, we outline the proof of~\cref{thm:bartlett}.
First, we define feedforward neural networks and show how simple GNNs can be interpreted as such.

\paragraph{Feedforward neural networks}
A \new{feedforward neural network} (FNN) is specified by a tuple $N=(\cN,\beta,\gamma)$
where $\cN$ describes the underlying \textit{architecture} and where $\beta$ and $\gamma$ define the \textit{parameters}
or \textit{weights}.
More specifically,
$\cN=\bigl(V^\cN,E^\cN,i_1,\ldots,\allowbreak i_p,o_1,\ldots,o_q,\alpha^\cN\bigr)$ where
$(V^\cN,E^\cN)$ is a \emph{finite} DAG with $p$ \new{input nodes} $i_1,\ldots,i_p$ of
in-degree $0$, and $q$ \new{output nodes} $o_1,\ldots,o_q$ of out-degree $0$. No other nodes have in- or out-degree zero.
Moreover, $\alpha^\cN$ is a function assigning to each node $v\in V^\cN\setminus\{i_1,\ldots,i_p\}$ an activation function $\alpha(v):\Rb\to\Rb$. Furthermore, the function $\beta \colon V^\cN\setminus\{i_1,\ldots,i_p\}\to\Rb$ is a function assigning biases to nodes, and finally, the function $\gamma \colon E^\cN\to\Rb$ assigns weights to edges. For an FNN $N$, we define its \new{size} $s$ as the number of biases and weights, that is $s=|V^\cN|-p+|E^\cN|$.

Given an FNN $N=(\cN,\beta,\gamma)$,
we get a function $\fnn_{N} \colon \Rb^p\to\Rb^q$ defined as follows. For all $v$ in $V^\cN$,
we define a function $h_v^{N} \colon \Rb^p\to\Rb$ such for $\ma=(a_1,\ldots,a_p)$ in $\Rb^p$,
\begin{equation*}
	h_v^{N}(\ma)\coloneqq\begin{cases}
		a_j     & \text{if $v=i_j$ for $j\in[p]$,} \\
		\alpha^\cN(v)\left(
		\sum_{u \in N^+(v)} \gamma(u,v)h_u^{N}(\ma) + \beta(v)
		\right) & \text{otherwise}.
	\end{cases}
\end{equation*}
Finally, $\fnn_{N} \colon \Rb^p\to\Rb^q$ is defined as $\ma\mapsto\fnn_{N}(\ma)\coloneqq\bigl(h_{o_1}^{N}(\ma),\ldots,h_{o_q}^{N}(\ma)\bigr)$.

\paragraph{Simple GNNs as FNNs}
We next connect simple GNNs in $\GNN_{\mathsf{slp}}(d,L)$ to FNNs. As described in Section~\ref{sec:simpleGNNs} such GNNs are
specified by $L+1$ activation functions $\pmb\sigma\coloneqq(\sigma_1,\ldots,\sigma_{L+1})$ and a weight vector $\Theta$ in $\Rb^{\nw}$ describing weight matrices and bias vectors in all the layers.
We show that for any attributed graph of order at most $n$ $\mG=(G,\mLG)$ in $\cG_{n,d}$ with $G=(V(G),E(G))$ and $\mLG$ in $\Rb^{n\times d}$ there exists an architecture $\cN_G(\pmb\sigma)$ such that for any weight assignment $\Theta$ in $\Rb^{\nw}$ of the GNN,
there exists $\beta_\Theta \colon V^{\cN_G}\to\Rb$ and $\gamma_\Theta \colon E^{\cN_G}\to\Rb$, satisfying
\begin{equation}\label{eq:gnn2fnn}
	\gnn_{\Theta,\pmb\sigma}(G,\mLG)=\fnn_{N=(\cN_G(\pmb\sigma),\beta_\Theta,\gamma_\Theta)}(\mLG') ,
\end{equation}
where $\mLG'$ in $\Rb^{nd}$ is the (column-wise) concatenation of the rows of the matrix $\mLG$. Moreover, $\cN_G(\pmb\sigma)$ is of polynomial size in the number of vertices and edges in $G$, the feature dimension $d$, and the number of layers $L$. Furthermore, $\cN_G(\pmb\sigma)$ has only a single output node $o$.

The idea behind the construction of $\cN_G(\pmb\sigma)$ is to consider the tree unraveling or unrolling, see, e.g.,~\cite{Morris2020b}, of the computation of $\gnn_{\Theta,\pmb\sigma}(G,\mLG)$ but instead of a tree we represent the computation more concisely as a DAG. The DAG $\cN_G(\pmb\sigma)$ is defined as follows.

\begin{itemize}\item The node set $V^{\cN_G}$ consists of the following nodes.
	      \begin{itemize}
		      \item  We have input nodes
		            $i_{v,j}$ for $v$ in $V(G)$ and $j$ in $[d]$ which will take the vertex labels $\emLG_{vj}$ in $\Rb$
		            as value.
		      \item  For each $t$ in $[L]$, we include the following nodes:  $n_{v,j}^{(t)}$ for $v$ in $V(G)$, $j$ in $[d]$.
		      \item Finally, we have a single output node $o$.
	      \end{itemize}
	      We thus have $d(L+1)|V(G)|+1$ nodes in total.
	\item The edge set $E^{\cN_{G}}$ consists of the following edges.
	      \begin{itemize}
		      \item We have edges encoding the adjacency structure of the graph $G$ in every layer. More specifically, we have an edge
		            $e_{u,j,v,k,t} \coloneqq (n_{u,j}^{(t-1)},n_{v,k}^{(t)})$ whenever $u$ in $N(v) \cup \{v\}$ and where $u$ and $v$ are in $V(G)$, $j$ and $k$ in $[d]$, and $t$ in $\{2,\ldots,L\}$.
		      \item We also have edges from the input nodes $i_{u,j}$ to $n_{v,k}^{(1)}$ for all $u$ in $N_G(v)\cup\{v\}$ and where $u$ and $v$ are in $V(G)$ and $j$ and $k$ in $[d]$.
		      \item Finally, we have edges connecting the last layer nodes to the output, i.e., edges
		            $e_{v,j} \coloneqq (n_{v,j}^{(L)},o)$ for all $v$ in $V(G)$ and $j$ in $[d]$.
	      \end{itemize}
	      We thus have $d|V(G)|+d^2((L-1)(E(G)+V(G)) + (E(G)+V(G)))$ edges in total.
	\item
	      Finally, we define the activation functions.
	      \begin{itemize}
		      \item $\alpha^\cN(n_{v,j}^{(t)}):=\sigma_t$ for  all $v$ in $V(G)$, $j$ in $[d]$ and  $t$ in $[L]$, and
		            $\alpha^\cN(o)\coloneqq \sigma_{L+1}$.
	      \end{itemize}
\end{itemize}
This fixes the architecture $\cN_G(\pmb\sigma)$. We next verify~\cref{eq:gnn2fnn}. Let $\Theta$ in $\Rb^{\nw}$ and $\mG=(G,\mLG)$ in $\cG^{n,d}$.
Let $\cN_G(\pmb\sigma)$ be the architecture defined above for the graph $G$. We define $\beta_\Theta$ and $\gamma_\Theta$, as follows.
\begin{itemize}
	\item  $\beta_\Theta \coloneqq V^{\cN_{G}}\to\Rb$ is such that $\beta_\Theta(n_{v,j}^{(t)})=b_j^{(t)}$ for all $v$ in $V(G)$, $j$ in $[d]$ and  $t$ in $[L]$. We also set $\beta_\Theta(o)=b$.
	\item  $\gamma_\Theta \colon E^{\cN_{G}}\to\Rb$ is such that $\gamma_\Theta(e_{u,j,v,k,t}) \coloneqq W_{jk}^{(2,t)}$ if $u \neq v$ and $\gamma_\Theta(e_{u,j,v,k,t}) \coloneqq W_{jk}^{(1,t)}$ otherwise, and $\gamma^\Theta(e_{v,j})=w_j$, for $u$ and $v$ in $V(G)$, $j$ and $k$ in $[d]$, and $t$ in $[L]$.
\end{itemize}
Note that we share weights across edges that correspond to the same edge in the underlying graph.

Now, if we denote by $\fb_{v}^{(t)}$ the feature vector in $\Rb^d$ computed in the $t$th layer by the GNN $\gnn_{\Theta,\pmb\sigma}(G,\mLG)$, then it is readily verified, by induction on the layers, that for $N = (\cN_G, \alpha_{\pmb\sigma},\beta_\Theta, \gamma_\Theta)$:
\begin{equation*}
	h^{N}_{n_{v,j}^{(t)}}=\fb_{v,j}^{(t)} \text{ and thus } h_o^{N} \coloneqq \sigma_{L+1}\left(\sum_{v\in V(G)}\sum_{j\in[d]} w_j \fb_{vj}^{(L)} + b\right),
\end{equation*}
from which~\cref{eq:gnn2fnn} follows.

We next expand the construction by obtaining an FNN that simulates GNNs on multiple input graphs. More specifically,
consider a set $\mathfrak G$ consisting of  $m$ graphs $\mG_1=(G_1,\mLG_1),\ldots,\mG_m=(G_m,\mLG_m)$ in $\cG_{n,d}$ and
a GNN in $\GNN_{\mathsf{slp}}(d,L)$ using activation functions $\pmb\sigma=(\sigma_1,\ldots,\sigma_{L+1})$
in its layers. We first construct an FNN architecture $\cN_{G_i}(\pmb\sigma)$ for each graph separately, as explained above, such that
for every $\Theta$ in $\Rb^{P}$, there exists $\beta_\Theta$ and $\gamma_\Theta$ such that
\begin{equation*}
	\gnn_{\Theta,\pmb\sigma}(G_i,\mL_i)=\fnn_{N_{G_i}\coloneqq(\cN_{G_i}(\pmb\sigma),\beta_\Theta,\gamma_\Theta)}(\mLG_i'),
\end{equation*}
with $\mLG_i'$ is the concatenation of rows in $\mL_i$, as before.

Then, let $\cN_{\mathfrak G}(\pmb\sigma)$ be the FNN architecture obtained as the disjoint union of $\cN_{G_1}(\pmb\sigma),\ldots,\cN_{G_m}(\pmb\sigma)$.
If we denote by $o_i$ the output node of $\cN_{G_i}(\pmb\sigma)$ in $\cN_{\mathfrak G}(\pmb\sigma)$, then we have again that for every $\Theta$ in $\Rb^{P}$, there exists $\beta_\Theta$ and $
	\gamma_\Theta$ such that
\begin{equation*}
	\gnn_{\Theta,\pmb\sigma}(\mG_i)=h^{N_{\mathfrak G}\coloneqq(\cN_{\mathfrak G}(\pmb\sigma),\beta_\Theta,\gamma_\Theta)}_{o_i}(\mLG')
\end{equation*}
for all $i$ in $[m]$, where $\mLG'\coloneqq (\mLG_1',\ldots,\mLG_m')$.

We recall that, for $t$ in $[L]$, the nodes in $\cN_{\mathfrak G}(\pmb\sigma)$ are of the form $\nu_{v,j}^{(t),g}$ for $v$ in $V_{G_g}$, $j$  in $[d]$ and $g$ in $[m]$.
In layer $L+1$, we have the output nodes $o_1,\ldots,o_m$. If the order of the graphs in $\mathfrak G$ is at most $n$, then every layer, except the last one, has $ndm$ nodes. The last layer only has $m$ nodes.

\paragraph{Piece-wise polynomial activation functions}
A \new{piece-wise polynomial activation function} $\sigma_{p,\delta}:\Rb\to\Rb$  is specified by  a partition of $\Rb$ into $p$ intervals $I_j$ and corresponding polynomials $p_j(x)$ of degree
at most $\delta$, for $j$ in $[p]$. That is, $\sigma_{p,\delta}(x)=p_j(x)$ if $x$ in $I_j$.
Examples of $\sigma_{p,\delta}(x)$ are:
$\mathsf{sign}(x) \colon \Rb\to\Rb \colon x\mapsto \mathbf{1}_{x\geq 0}$ for which $p=2$ and $\delta=0$,
$\mathsf{relu}(x)  \colon \Rb\to\Rb \colon x\mapsto \max(0,x)$ for which $p=2$ and $\delta=1$, and $\mathsf{lsig}(x)  \colon \Rb\to\Rb \colon x \mapsto \max(0,\min(1,x))$ for which $p=3$ and
$\delta=1$. \emph{Piece-wise linear activation functions} are of the form $\sigma_{p,1}$, i.e., they are defined in terms of linear polynomials. The parameters of an activation function $\sigma_{p,\delta}$ consist of the coefficients of the polynomials involved and the boundary  points (numbers) of the intervals in the partition of $\Rb$.

\paragraph{Proof of~\cref{thm:bartlett}}
We next derive upper bounds on the VC dimension of GNNs by the approach used in \citet{Bar+2019},
where they used it for bounding the VC dimension of FNNs using piecewise polynomial activation functions. Their approach allows for recovering known bounds on the VC dimension of FNNs in a unified manner. As we will see, the bounds by~\citet{Bar+2019} for FNNs naturally translate to bounds for GNNs.

Assume $d$ and $L$ in $\Nb$. In this section, we will consider the subclass of GNNs in  $\GNN_{\mathsf{slp}}(d,L)$ that use
\textit{piece-wise polynomial activation functions} with $p > 0$ pieces and degree $\delta \geq 0$.
As explained in  Section~\ref{sec:simpleGNNs}, $\nw$ is the total number of
(learnable) parameters for our GNNs in $\GNN_{\mathsf{slp}}(d,L)$. As shorthand notation, we define
$P\coloneqq \nw$. We first bound $\vcdim_{\cG_{d,n}}\bigl(\GNN_{\mathsf{slp}}(d,L)\bigr)$ and then use this bound to obtain a bound for
$\vcdim_{\cG_{d,\leq u}}\bigl(\GNN_{\mathsf{slp}}(d,L)\bigr)$.

We take $\mathfrak G$ consisting of  $m$ graphs $\mG_1=(G_1,\mLG_1),\ldots,\mG_m=(G_m,\mLG_m)$ in $\cG_{n,d}$ and
consider the FNN architecture $\cN_{\mathfrak G}(\pmb\sigma)$ define above with output nodes $o_1,\ldots,o_m$.
Let $\mathsf{tresh}:\Rb\to\Rb$ such that $\mathsf{tresh}(x)=1$ if $x\geq 2/3$ and $\mathsf{tresh}(x)=0$ if $x\leq 1/3$.
We will bound
\begin{equation*}
	K'\coloneqq \big|\bigl\{ \bigl(\mathsf{tresh}(h^{N_{\mathfrak{G}}}_{o_1}(\mLG')),\ldots,
	\mathsf{tresh}(h^{N_{\mathfrak{G}}}_{o_m}(\mLG'))\bigr) \colon N_{\mathfrak{G}}\coloneqq(\cN_\mathfrak{G}(\pmb\sigma),\beta_\Theta,\gamma_\Theta), \Theta\in\Rb^P\bigr\}\big|,
\end{equation*}
as this number describes how many $0/1$ patterns can occur when $\Theta$ ranges over $\Rb^P$. These $0/1$ patterns correspond, by the construction of $N_{\mathfrak{G}}$ and the semantics of its output nodes, to how many of the input graphs in $\mathfrak G$ can be shattered.
To bound $K'$ using the approach in  \citet{Bar+2019} we need to slightly change the activation function $\sigma_{L+1}$ used in the FNN architecture.
The reason is that Bartlett et al. use the $\mathsf{sign}$ function to turn a real-valued function into a $0/1$-valued function. In contrast, we use
the $\mathsf{tresh}$ function described above.

Let $\pmb\sigma'\coloneqq (\sigma_1,\ldots,\sigma_L,\sigma_{L+1}-1/3)$. We will bound $K'$ by bounding
\begin{equation*}
	K\coloneqq\big|\bigl\{ \bigl(\mathsf{sign}(h^{N_{\mathfrak{G}}}_{o_1}(\mLG')),\ldots,
	\mathsf{sign}(h^{N_{\mathfrak{G}}}_{o_m}(\mLG'))\bigr) \colon N_{\mathfrak{G}}\coloneqq(\cN_\mathfrak{G}(\pmb\sigma'),\beta_\Theta,\gamma_\Theta), \Theta\in\Rb^P\bigr\}\big|.
\end{equation*}
Note that $K'\leq K$ because if $\mathsf{tresh}(\sigma_{L+1})(\mathbf{x})=1$ then $\sigma_{L+1}(\mathbf{x})\geq 2/3$ and hence $\sigma_{L+1}(\mathbf{x})-1/3>0$
and hence $\mathsf{sign}(\sigma_{L+1}(\mathbf{x})-1/3)=1$. Similarly, $\mathsf{tresh}(\sigma_{L+1})(\mathbf{x})=0$ then $\sigma_{L+1}(\mathbf{x})\leq 1/3$ and hence $\sigma_{L+1}(\mathbf{x})-1/3\leq 0$
and hence $\mathsf{sign}(\sigma_{L+1}(\mathbf{x})-1/3)=0$.

Then, if $\vcdim_{\cG_{d,n}}\bigl(\GNN_{\mathsf{slp}}(d,L)\bigr)=m$ then $K \geq 2^m$. We thus bound $K$ in terms of a function $\kappa$ in $m$ and then use $2^m\leq \kappa(m)$ to find an upper bound for $m$, i.e.,  an upper bound for $\vcdim_{\cG_{d,n}}\bigl(\GNN_{\mathsf{slp}}(d,L)\bigr)$.
To bound $K$ we can now use the approach of \citet{Bar+2019}.
In a nutshell, the entire parameter space $\Rb^P$ is partitioned into pieces $S_1,\ldots,S_\ell$
such that whenever $\Theta$ and $\Theta'$ belong to the same piece (i)~they incur the same sign pattern in $\{0,1\}^m$; and (ii)~each $h^{N_{\mathfrak{G}}}_{o_1}(\mLG')$ is a polynomial of degree at most $1+L\delta^{L}$. For $\delta=0$, these are polynomials in $d+1$ variables,
for $\delta>0$, the number of variables is $P$. Crucial in Bartlett's approach is the following lemma.

\begin{lemma}[Lemma 17 in \citet{Bar+2019}]\label{lem:lemma17bartlett}
	Let $p_1(\mathbf{x}),\ldots,p_r(\mathbf{x})$
	be polynomials of degree at most $\delta$ and in variables $\mathbf{x}$ satisfying $|\mathbf{x}|\leq r$, where $|\cdot|$ denotes the number of components of a vector.	Then
	\begin{equation*}
		\Big|\bigl\{\bigl(\mathsf{sign}(p_1(\Theta)),\ldots,\mathsf{sign}(p_r(\Theta))\bigr) \mid \Theta\in\Rb^{|\mathbf{x}|}\bigr\}\Big|\leq 2\left(\frac{2e r \delta}{|\mathbf{x}|}\right)^{|\mathbf{x}|}.
	\end{equation*}
\end{lemma}

Given property (ii)~of the pieces $S_1,\ldots,S_\ell$, we can apply the above lemma to the polynomials $h^{N_{\mathfrak{G}}}_{o_1}(\mLG'),\ldots,h^{N_{\mathfrak{G}}}_{o_m}(\mLG')$ and, provided that the number of variables is at most $m$, obtain a bound for $K$ by $\ell 2 \left(\frac{2em}{d+1}\right)^{d+1}$, when $\delta=0$, and
$K\leq \ell  2 \left(\frac{2em(1+L\delta^{L})}{P}\right)^P$ for $\delta>0$.

It then remains to bound the
number of parts $\ell$. Bartlett et al. show how to do this inductively (on the number of layers), again using Lemma~\ref{lem:lemma17bartlett}. More precisely, every node in the FNN architecture is associated with a number of polynomials. In layer $t$ we have $nmd$ nodes (number of computation nodes), and we associate with each node $p$ polynomials (number of breakpoints of activation function) of degree at most  $1+(t-1)\delta^{t-1}$ and have $(2d+1)d$ variables for $\delta=0$  and $(2d+1)dt$ variables for $\delta>0$. We then
get, for $\delta=0$,
\begin{equation}\label{eq:kbounddegzero}
	K\leq  2^L\left(\left(\frac{2edmnp}{(2d+1)d}\right)^{(2d+1)d}\right)^L 2 \left(\frac{2em}{d+1}\right)^{d+1},
\end{equation}
and for $\delta>0$,
\begin{equation}\label{eq:kbounddegnotzero}
	K\leq \prod_{t=1}^L 2\left(\frac{2edmnp(1+(t-1)\delta^{t-1})}{(2d+1)dt}\right)^{(2d+1)dt} 2 \left(\frac{2em(1+L\delta^{L})}{P}\right)^P.
\end{equation}
These are precisely the bounds given in \citet{Bar+2019} applied to our FNN.
It is important, however, to note that this upper bound is only valid when~\cref{lem:lemma17bartlett} can be applied, and hence, the number of variables must be smaller than the number of polynomials. For $t$ in $[L]$, we must have that the number of variables is less than the number of polynomials. We have $nmdp$ polynomials, and up to layer $t$ we have $(2d+1)dt$ parameters (variables). Hence, we must have
$(2d+1)dt\leq nmdp$ or $(2d+1)t\leq nmp$, and also $D\leq m$ or $P \leq nmdp$ for $\delta>0$. For $\delta=0$, we need $(2d+1)d\leq nmdp$ (or $(2d+1)\leq nmp$) and $d+1\leq m$. The following conditions are sufficient
\begin{equation*}
	P\leq m  \text{ for $\delta>0$, and } 2d+1\leq m \text{ for $\delta=0$.}  \tag{$\dagger$}
\end{equation*}

\paragraph{FNN size reduction based on $\wlone{}$}
We next bring in \wlone{} into consideration by collapsing computation nodes in $N_{\mathfrak G}$ in each layer based on their equivalence with regards to $\wlone{}$. In other words, if we assume that the graphs to be shattered have at most $u$ vertex colors, then we have at most $u$---rather than $n$---computation nodes per graph. This implies that the parameter $n$ in the above expression can be replaced by $u$.

As a consequence, following \citet{Bar+2019}, using the weighted AM–GM inequality on the right-hand side of the inequalities~(\ref{eq:kbounddegzero}) and~(\ref{eq:kbounddegnotzero}), we obtain a bound for the VC dimension by finding maximal
$m$ satisfying,
for $\delta=0$,
\begin{equation*}
	2^m\leq 2^{L+1} \left(\frac{2ep (udL+1)m}{P}\right)^{P}
\end{equation*}
and for $\delta>0$,
\begin{equation*}
	2^m\leq 2^{L+1} \left(\frac{2ep \bigl(ud\sum_{t=1}^L (1+(t-1)\delta^{t-1})+(1+L\delta^L)\bigr)m}{\frac{L(L+1)}{2}(2d+1)d+P}\right)^{\frac{L(L+1)}{2}(2d+1)d+P}.
\end{equation*}
Such $m$ are found in \cite{Bar+2019}, resulting in the following bounds.
\begin{proposition}[\cite{Bar+2019} modified for our FNN $N_{\mathfrak G}$]
	\begin{equation*}
		\vcdim_{\cG_{d,\leq u}}\bigl(\GNN_{\mathsf{slp}}(d,L)\bigr)\leq
		\begin{cases}
			\mathcal O(Ld^2\log(p(udL+1)))                      & \text{if $\delta=0$}  \\
			\mathcal O(L^2d^2\log(p(udL+1)))                    & \text{if $\delta=1$}  \\
			\mathcal O(L^2d^2\log(p(udL+1))+L^3d^2\log(\delta)) & \text{if $\delta>1$}. \\
		\end{cases}
	\end{equation*}
\end{proposition}
We can simplify this to $\mathcal O(P\log(puP))$, $\mathcal O(LP\log(puP))$, and $\mathcal O(LP\log(puP)+\allowbreak L^2P\allowbreak \log(\delta))$, respectively. We note that since these are larger than $P$, the condition ($\dagger$) is satisfied.

\section{Additional experimental data and results}
Here, we report on additional experimental details, results, and dataset generation.

\subsection{Simple GNN layer used for \textbf{Q1} and \textbf{Q2}}\label{simple_app}
The simple GNN layer used in \textbf{Q1} and \textbf{Q2} updates the feature of vertex $v$ at layer $t$ via
\begin{equation}\label{gnnsimple}
	\fb_{v}^{(t)}\coloneqq \sigma\Big(\fb_{v}^{(t-1)}\mW_1^{(t)} +
	\sum_{u\in N(v)}\fb_{u}^{(t-1)}\mW_2^{(t)} \Big)\in\Rb^d,
\end{equation}
where $\mW^{(t)}_1$ and $\mW_2^{(t)}  \in \Rb^{d \times d}$ are parameter matrices. In the experiments, we used $\relu$ activation functions.

\subsection{GNN architecture used for \textbf{Q3}}\label{syn_gnn_app}
For the experiments on \textbf{Q3} we extend the simple GNN layer from \cref{gnnsimple} used in \textbf{Q1} and \textbf{Q2}. We update the feature of vertex $v$ at layer $t$ via
\begin{equation}\label{gnnsimpleplus}
	\fb_{v}^{(t)}\coloneqq \sigma\Big(\text{BN}\Big(\fb_{v}^{(t-1)}\mW^{(t)} + \sum_{u\in N(v)}\mathsf{mlp}^{(t)}(\fb_{u}^{(t-1)}) \Big)\Big) \in\Rb^d,
\end{equation}
where BN is a batch normalization module \cite{ioffe2015batch} and $\mathsf{mlp}^{(t)}$ is a two-layer perceptron with architecture,
\begin{equation*}
	\text{Linear} \rightarrow \text{BN} \rightarrow \relu \rightarrow \text{Linear}.
\end{equation*}
We, therefore, use a normalized 2-layer MLP to generate messages in each layer.
We found this change necessary to ensure smooth convergence on the challenging synthetic task posed by \textbf{Q3}, where the GNN has to memorize an arbitrary binary graph labeling. Moreover, in the experiments, we used $\relu$ activation functions.

\subsection{Synthetic dataset generation}\label{syn_app}
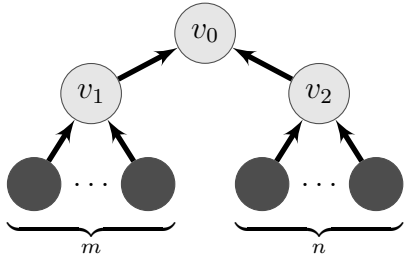
\begin{figure}[t]
	\centering
	\input{figures/trees/tree_figure.tex}
	\caption{A visualization of a $T_{m,n}$.}
	\label{Fig:Tree}
\end{figure}
To address \textbf{Q3}, we aim to empirically estimate how well GNNs of different sizes can fit arbitrary binary labelings of graphs.
We construct a synthetic dataset that focuses on a simple class of trees.
Formally, for two natural numbers $m$ and $n$ in $\mathbb{N}_{\geq0}$, we define the graph $T_{m,n}=(V_{m,n},E_{m,n})$ as a directed tree with vertex set $V=\{v_0,\dots,v_{m+n+3}\}$.
The root $v_0$ has two children $v_1$ and $v_2$. The remaining $m+n$ vertices form the leaves such that vertex $v_1$ has $m$ children and $v_2$ has $n$ children.
\Cref{Fig:Tree} provides a visual example.

For a chosen $k$ in $\mathbb{N}, k\geq4$, we define:
\begin{equation*}
	\mathcal{T}_k = \Big\{T_{m,n} \:\mid\: 0 \leq m \leq \Big\lfloor\frac{(k-3)}{2}\Big\rfloor, n = k - 3 -m \Big\}.
\end{equation*}
Therefore, $\mathcal{T}_k$ contains all distinct graphs $T_{m,n}$ with $|V_{m,n}(T_{m,n})|=k$.
In particular, we observe $|\mathcal{T}_k| = \lfloor\frac{(k-3)}{2}\rfloor$.

For $k \in \{10,20,\dots,90\}$, we aim to test how well a GNN with a given feature dimension $d$ in $\{4, 16, 64, 256\}$ can learn binary labelings $y \colon \mathcal{T}_k \rightarrow \{0, 1\}$ of $\mathcal{T}_k$.
The labeling $y$ is obtained by sampling the label $y(T)$ uniformly at random for all $T$ in $\mathcal{T}_k$.
We then train a GNN model with stochastic gradient descent to minimize the binary cross entropy on the dataset $(\mathcal{T}_k, y)$.
For each combination of $k$ and $d$, we repeat the experiment $50$ times.
We resample a new labeling $y$ and a new random initialization of the GNN model in each repetition.

\subsection{Simulating bitlength via higher feature dimension}\label{bit_dim}
Here, we outline how to simulate a higher bitlength via a higher feature dimension. Assume the simple GNN layer of~\cref{gnnsimple}. Clearly, we can express the matrices $\mW_1^{(t)}$ and $\mW_2^{(t)}$ as the sum of $k$ matrix with smaller bitlength, e.g.,
\begin{equation*}
	\mW_2^{(t)} = \mW_2^{(1,t)} + \cdots +\mW_2^{(k,t)}.
\end{equation*}
Hence, we can re-write the aggregation in~\cref{gnnsimple} as
\begin{equation*}
	\sum_{u\in N(v)}\fb_{u}^{(t-1)} \big[\mW_2^{(1,t)}, \cdots, \mW_2^{(k,t)} \big] \cdot \vec{M} \in\Rb^d,
\end{equation*}
where $[\cdots]$ denotes column-wise matrix concatenation and $\vec{M}$ in $\{ 0, 1\}^{kd \times d}$ is a matrix such that
\begin{equation*}
	\sigma\Big(\fb_{v}^{(t-1)}\mW_1^{(t)} + \sum_{u\in N(v)}\fb_{u}^{(t-1)} \big[\mW_2^{(1,t)}, \cdots, \mW_2^{(k,t)} \big] \cdot \vec{M}\Big)  = \sigma\Big(\fb_{v}^{(t-1)}\mW_1^{(t)} +
	\sum_{u\in N(v)}\fb_{u}^{(t-1)}\mW_2^{(t)} \Big),
\end{equation*}
i.e., the matrix $\vec{M}$ sums together columns of the aggregated features such that they have feature dimension $d$.

\begin{table}[t!]
	\begin{center}
		\caption{Dataset statistics and properties.}
		\resizebox{1.0\textwidth}{!}{ 	\renewcommand{\arraystretch}{1.05}
			\begin{tabular}{@{}lcccccc@{}}\toprule
				\multirow{3}{*}{\vspace*{4pt}\textbf{Dataset}} & \multicolumn{6}{c}{\textbf{Properties}}                                                                                                                         \\
				\cmidrule{2-7}
				                                               & Number of  graphs                       & Number of classes/targets & $\varnothing$ Number of nodes & $\varnothing$ Number of edges & Node labels & Edge labels \\ \midrule
				$\textsc{Enzymes}$                             & 600                                     & 6                         & 32.6                          & 62.1                          & \cmark      & \xmark      \\
				$\textsc{MCF-7}$                               & 27\,770                                 & 2                         & 26.4                          & 28.5                          & \cmark      & \cmark      \\
				$\textsc{MCF-7H}$                              & 27\,770                                 & 2                         & 47.3                          & 49.4                          & \cmark      & \cmark      \\
				$\textsc{Mutagenicity}$                        & 4\,337                                  & 2                         & 30.3                          & 30.8                          & \cmark      & \cmark      \\
				$\textsc{NCI1}$                                & 4\,110                                  & 2                         & 29.9                          & 32.3                          & \cmark      & \xmark      \\
				$\textsc{NCI109}$                              & 4\,127                                  & 2                         & 29.7                          & 32.1                          & \cmark      & \xmark      \\
				\bottomrule
			\end{tabular}}
		\label{statistics}
	\end{center}
\end{table}

\end{document}

%% file: macros.tex
\newcommand{\mG}{\mathbf{G}}
\newcommand{\mH}{\mathbf{H}}
\newcommand{\mL}{\mathbf{L}}

\newcommand{\mW}{\mathbf{W}}
\newcommand{\ma}{\mathbf{a}}
\newcommand{\mb}{\mathbf{b}}
\newcommand{\mw}{\mathbf{w}}

\newcommand{\mLG}{\mathbf{L}}
\newcommand{\emLG}{L}
\newcommand{\inid}{d}
\newcommand{\cC}{\mathcal{C}}
\newcommand{\cD}{\mathcal{D}}
\newcommand{\cF}{\mathcal{F}}
\newcommand{\cG}{\mathcal{G}}

\newcommand{\cN}{\mathcal{N}}
\newcommand{\cO}{\mathcal{O}}

\newcommand{\cS}{\mathcal{S}}
\newcommand{\cT}{\mathcal{T}}

\newcommand{\cX}{{\mathcal X}}

\newcommand{\Nb}{\mathbb{N}}

\newcommand{\Rb}{\mathbb{R}}

\newcommand{\fnn}{\mathsf{fnn}}
\newcommand{\gnn}{\mathsf{gnn}}

\newcommand{\GNN}{\mathsf{GNN}}

\newcommand{\bool}{\mathbb{B}}
\newcommand{\vcdim}{\ensuremath{\mathsf{VC}\text{-}\mathsf{dim}}}

\newcommand{\kwl}{$k$\textrm{-}\textsf{WL}\xspace}
\newcommand{\wlone}{$1$\textrm{-}\textsf{WL}}

\newcommand{\fb}{\mathbf{f}}
\newcommand{\hb}{\mathbf{h}}

\newcommand{\Nbb}{\mathbb{N}}

\newcommand{\UPD}{\mathsf{UPD}}
\newcommand{\AGG}{\mathsf{AGG}}
\newcommand{\RO}{\mathsf{READOUT}}

\newcommand{\REL}{\mathsf{RELABEL}}

\newcommand{\relu}{\mathsf{reLU}}

\newcommand{\new}[1]{\emph{#1}}

\renewcommand{\vec}[1]{\mathbf{#1}}
\newcommand{\oms}{\{\!\!\{}
\newcommand{\cms}{\}\!\!\}}
\newcommand{\tup}[1]{{(#1)}}

\newcommand{\Bigmid}{\mathrel{\Big|}}
\newcommand{\Mat}[1]{\begin{pmatrix}#1\end{pmatrix}}

\newcommand{\nw}{d(2dL+L+1) + 1}

%% file: figures/trees/tree_figure.tex
\begin{tikzpicture}[
nodes={
minimum width=20pt,
font=\large}
]
    \node at (0,0) [circle, fill=black!10, draw=black!70] (v0) {$v_0$};
    \node at (-1.5,-0.8) [circle, fill=black!10, draw=black!70] (v1) {$v_1$};
    \node at (1.5,-0.8) [circle, fill=black!10, draw=black!70] (v2) {$v_2$};
    \draw[line width=2pt, -latex'] (v1) -- (v0);
    \draw[line width=2pt, -latex'] (v2) -- (v0);

    \node at (-2.25,-2) [circle, fill=black!70, draw=black!70] (v3) {};
    \node at (-0.75,-2) [circle, fill=black!70, draw=black!70] (v4) {};
    \node at (-1.5,-2) [] (d1) {$\mathbf{\dots}$};
    \node at (-1.5,-2.7) [] (b2) {$\underbrace{\hspace{2.2cm}}_{m}$};
    \draw[line width=2pt, -latex'] (v3) -- (v1);
    \draw[line width=2pt, -latex'] (v4) -- (v1);

    \node at (2.25,-2) [circle, fill=black!70, draw=black!70] (v5) {};
    \node at (0.75,-2) [circle, fill=black!70, draw=black!70] (v6) {};
    \node at (1.5,-2) [] (d2) {$\mathbf{\dots}$};
    \node at (1.5,-2.7) [] (b2) {$\underbrace{\hspace{2.2cm}}_{n}$};
    \draw[line width=2pt, -latex'] (v5) -- (v2);
    \draw[line width=2pt, -latex'] (v6) -- (v2);
\end{tikzpicture}